%% file: main.tex
\documentclass[letterpaper, 10 pt, conference]{ieeeconf}  

\IEEEoverridecommandlockouts                              

\newcommand{\ARXIV}[2]{#1} 

\usepackage{graphicx}

\usepackage{amsthm}
\usepackage{amsmath}
\usepackage{amssymb}
\usepackage{newtxtext,newtxmath}
\usepackage{mathtools}
\usepackage{amssymb}
\usepackage{booktabs}
\usepackage{multirow}
\usepackage{tabularx}
\usepackage[table,RGB]{xcolor}
\usepackage{balance}
\usepackage{placeins}
\usepackage{subcaption}
\usepackage{microtype}
\usepackage{paralist}
\usepackage{siunitx}
\usepackage{array}
\usepackage{makecell}
\usepackage{url}
\usepackage{fontawesome5}
\usepackage{cleveref}
\usepackage[T1]{fontenc}

\usepackage{tikz}
\usepackage{tikz-3dplot}
\usepackage{pgfplots}
\usepackage{pgf}
\usepackage{pgfplots}
\usepackage{pgfplotstable}
\usepackage{etoolbox}

\usepackage[ruled,vlined]{algorithm2e}
\usepackage{amsmath}
\usepackage{bm}

\usepackage{styles}

\usetikzlibrary{shapes,shadings,positioning}

\usepackage[hang,flushmargin]{footmisc}
\usepackage[skip=3pt]{caption}
\title{
\LARGE \bf
EZ-SP: Fast and Lightweight Superpoint-Based 3D Segmentation
}

\author{
  Louis Geist\textsuperscript{1} \quad
  Loic Landrieu\textsuperscript{1} \quad
  Damien Robert\textsuperscript{2} \\[4pt]
  \textsuperscript{1} LIGM, ENPC, IP Paris, Univ Gustave Eiffel, CNRS, Marne-la-Vallée, France \\ 
  \textsuperscript{2} DM$3$L, University of Zurich, Zurich, Switzerland
}
\begin{document}

\maketitle
\thispagestyle{empty}
\pagestyle{empty}

\begin{abstract}
Superpoint-based pipelines provide an efficient alternative to point- or voxel-based 3D semantic segmentation, but are often bottlenecked by their CPU-bound partition step.  
We propose a \textbf{learnable, fully GPU} partitioning algorithm that generates geometrically and semantically coherent superpoints \OVERSEGSPEEDUPVSPCP faster than prior methods.  
Our module is compact (under 60k parameters), trains in under 20 minutes with a differentiable surrogate loss, and requires no handcrafted features.  
Combined with a lightweight superpoint classifier, the full pipeline fits in $<\!2$\,MB of VRAM, scales to multi-million-point scenes, and supports real-time inference.  
With \THROUGHPUTVSPTVTHREE faster inference and \SIZEVSPTVTHREE fewer parameters, \MODELNAME~matches the accuracy of point-based SOTA models across three domains: indoor scans (S3DIS), autonomous driving (KITTI-360), and aerial LiDAR (DALES). 
Code and pretrained models are accessible at \GITHUB{}.
\end{abstract}

\section{Introduction}

\input{sections/1_intro}
\section{Related Work}
\input{sections/2_related}
\section{Method}
\input{sections/3_method2}
\section{Experiments}
\input{sections/4_experiments}
\section{Conclusion}
We presented \MODELNAME, a fast and lightweight superpoint partitioning model that removes the long-standing partitioning bottleneck in superpoint-based 3D semantic segmentation.  
Our end-to-end pipeline 
can process $1.7$M points/s on a single consumer grade GPU and achieves near–state-of-the-art accuracy in indoor, terrestrial, and aerial LiDAR benchmarks.  
By drastically lowering runtime costs, \MODELNAME~opens the door to efficient analysis of massive 3D scans and real-time perception on resource-constrained and embedded platforms.  

\ARXIV{\input{sections/5_acknowledgements}}{}

\vspace{-0mm}



\balance
\input{main.bbl}

\ARXIV{
\renewcommand*{\thesection}{A.\arabic{section}}
\setcounter{figure}{0}
\setcounter{prop}{0}
\renewcommand\thefigure{A.\arabic{figure}}
\setcounter{table}{0}
\renewcommand\thetable{A.\arabic{table}}
    \input{sections/appendix}
}{}

\end{document}

%% file: sections/1_intro.tex
Accurate 3D semantic segmentation is critical for robotic perception, enabling tasks such as autonomous driving~\cite{loiseau2022online}, navigation~\cite{xu2015real,busch2025one}, and mapping~\cite{pfaff2007towards}.  
However, balancing accuracy with computational efficiency remains a major challenge.  
Real-world point clouds often contain millions of points and must be processed under strict latency constraints---for instance, rotating automotive LiDAR  typically acquires 1.3M points per second.  
Yet, state-of-the-art (SOTA) models routinely exceed ten million parameters, rely on costly test-time augmentation to reach their performance, and can only process small scenes at a time.  
This gap limits deployment in real-time robotic systems, as well as in AR/VR~\cite{makhataeva2020augmented} and large-scale smart city applications~\cite{beigi2017real}.

A common strategy to reduce complexity is to group points into \emph{superpoints}~\cite{landrieu2018spg,papon2013voxel}: spatially contiguous, geometrically homogeneous regions.  
Reasoning on superpoints rather than individual points dramatically reduces memory and computation, while maintaining SOTA or near-SOTA accuracy~\cite{hui2021superpoint,robert2023efficient}. 
However, the partition stage remains a critical bottleneck: often CPU-bound, slow to tune (each parameter sweep may take hours), and requires complex optimization over handcrafted geometric and radiometric features.
This overhead in runtime and engineering effort has hampered the broader adoption of superpoint-based methods.  

\begin{figure}[t]
\centering
\input{figures/teaser3}
\vspace{-0mm}
\caption{
\textbf{Inference Speed \vs Performance \vs Model Size.}  
Comparison of end-to-end pipelines (preprocessing to inference) on S3DIS.  
\MODELNAME~achieves near-SOTA accuracy with only $400$k parameters, while being orders of magnitude faster than point-based networks, and the only method to match the acquisition rate of automotive LiDAR (\faCar).  
}
\label{fig:teaser}
\end{figure}
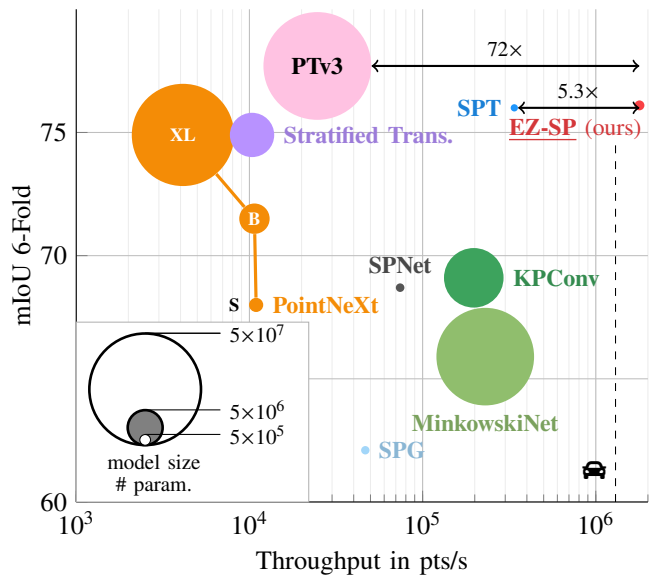

To overcome these limitations, we introduce \MODELNAME~ (easy-superpoints): a lightweight, \emph{learnable, fully GPU} pipeline to partition and segment raw point clouds into superpoints \emph{on the fly}.  
A $60$k-parameter backbone produces low-dimensional embeddings optimized to detect \emph{semantic transitions} directly from raw point clouds and without handcrafted features.  
We then use a massively parallel clustering algorithm to produce semantically homogeneous, geometrically simple superpoints---\PARTITIONALGOSPEEDUPVSPCP faster than the widely used Parallel Cut Pursuit~\cite{raguet2019parallel,landrieu2017cut}.  
These superpoints are then processed by a $330$k-parameter superpoint segmentation network, yielding dense labels at a fraction of the cost of point- or voxel-based approaches.  
As shown in \cref{fig:teaser}, our method matches or surpasses SOTA accuracy while delivering over \THROUGHPUTVSPTVTHREE faster end-to-end inference than PointTransformer-v3~\cite{wu2024point}.  

\para{Key advantages and contributions.} Our approach is:
\begin{compactitem}
    \item \textbf{Fast}: partitioning a raw point cloud is over \OVERSEGSPEEDUPVSPCP faster than the fastest graph-based approaches; end-to-end inference is over \THROUGHPUTVSPTVTHREE faster than recent SOTA models.  
    \item \textbf{Lightweight}: the entire model fits in $<\!2$\,MB VRAM and trains from scratch in $3$\,h on a single A6000.
    \item \textbf{Versatile}: The same configuration generalizes across indoor (S3DIS), mobile (KITTI-360), and aerial mapping (DALES) with minimal hyperparameter tuning. Our model runs in real time on small scans ($>$1.3M pt$/$s) and scales to city-scale point clouds in a single pass.
    \item \textbf{Scalable}: GPU-only segmentation of tens of millions of points in a single pass on a single consumer-grade GPU.  
\end{compactitem}

%% file: figures/teaser3.tex
\begin{tikzpicture}
\begin{semilogxaxis}[%
    width=.88\linewidth,
    scale only axis,
    ytick={60,65,70,75},
    yticklabels={60,,70,75},
    xtick={1000,10000,100000,1000000, 10000000},
    xmin=1000,
    xmax=2000000,
    ymin=60,
    ymax=80,
    axis x line*=bottom,
    axis y line*=left,
    grid={major},
    xminorgrids=true,
    minor grid style={gray!15,line width=.25pt},
    legend pos=south east,
    ylabel={mIoU 6-Fold},
    xlabel={Throughput in pts/s},
    ylabel style={xshift=-0cm,yshift=-0.5cm},
    clip marker paths=false,
    clip mode=individual,
    ]
]


\node[draw=none,right,fill=EZSPCOLOR, circle, minimum size = 30*sqrt(0.39), anchor=center,thin, scale=0.2] (ezsp) at (axis cs:1787432,76.1) {};
\node[draw=none, text=EZSPCOLOR!85!black, below left=0mm and -2mm of ezsp] {\small{\bf \underline{\MODELNAME}} (ours)};

\node[draw=none,right,fill=SPTCOLOR, circle, minimum size = 30*sqrt(0.21), anchor=center,thin, scale=0.2] (spt) at (axis cs:338461,76.0) {};
\node[draw=none, text=SPTCOLOR!85!black, left=0mm of spt] {\small{\bf SPT}};


\node[draw=none,right,fill=SPGCOLOR, circle, minimum size = 30*sqrt(0.280)-0 pt, anchor=center,thin, scale=0.2] (spg) at (axis cs:46676,62.1) {};
\node[draw=none, text=SPGCOLOR!85!black, right= 0mm of spg] {\small{\bf SPG}};


\node[draw=none,right,fill=POINTTRANSCOLOR, circle, minimum size = 30*sqrt(46.2)-0pt, anchor=center,line width=1pt, scale=0.2] (pointtrans) at (axis cs:24670,77.7) {};
\node[draw=none, text=black, fill=none] at (pointtrans.center) {\small{\bf PTv3}};

\node[draw=none,right,fill=KPCONVCOLOR, circle, minimum size = 30*sqrt(14.1)-0pt, anchor=center,line width=1pt, scale=0.2] (kpconv) at (axis cs:197456.,69.1) {};
\node[draw=none, text=KPCONVCOLOR!85!black, right= 0mm of kpconv] {\small{\bf KPConv}};

\node[draw=none,right,fill=MINKOCOLOR, circle, minimum size = 30*sqrt(37.9)-0pt, anchor=center,line width=1pt, scale=0.2] (minko) at (axis cs:230006,65.9) {};
\node[draw=none, text=MINKOCOLOR!85!black, below= 0mm of minko] {\small{\bf MinkowskiNet}};

\node[draw=none,right,fill=POINTNEXTCOLOR, circle, minimum size = 30*sqrt(0.8)-0pt, anchor=center,line width=1pt, text=white, scale=0.2] (ptna) at (axis cs:10931 ,68) {};
\node[draw=none,right,fill=POINTNEXTCOLOR, circle, minimum size = 30*sqrt(3.8)-1pt, anchor=center,line width=1pt, text=white, scale=0.2] (ptnb) at (axis cs:10689 ,71.5) {};
\node[draw=none,right,fill=POINTNEXTCOLOR, circle, minimum size = 30*sqrt(41.6)-0pt, anchor=center,line width=1pt, text=white, scale=0.2] (ptnd) at (axis cs:4131 ,74.9) {};

\node[draw=none, text=POINTNEXTCOLOR, right= 0mm of ptna] {\small{\bf PointNeXt}};

\node[black, left=0mm of ptna] {\scriptsize{\bf {S}}};
\node[white] at (ptnb.center) {\scriptsize{\bf B}};
\node[white] at (ptnd.center) {\scriptsize{\bf XL}};

\draw [very thick, POINTNEXTCOLOR] (ptna) -- (ptnb)  -- (ptnd);

\node[draw=none,right,fill=STRATTRANSCOLOR, circle, minimum size = 30*sqrt(8)-1pt, anchor=center,line width=1pt, scale=0.2] (strat) at (axis cs:10352,74.9) {};
\node[draw=none, text=STRATTRANSCOLOR!85!black, right= 0mm of strat] {\small{\bf Stratified Trans.}};

\node[draw=none,right,fill=SPNETCOLOR, circle, minimum size = 30*sqrt(0.33)-1pt, thin, anchor=center,line width=1pt, scale=0.2] (spnet) at (axis cs:74153,68.7) {};
\node[draw=none, text=SPNETCOLOR!85!black, above= 0mm of spnet] {\small{\bf SPNet}};

\draw [dashed] (axis cs:1300000,60) -- (axis cs:1300000,74.5);
\node [draw=none, anchor=east] at (axis cs:1300000,61.3) {\small \faCar}; 

\draw [fill=black!0!white, draw=black!30] (axis cs:1000, 60.0) rectangle (axis cs:22000,67.3);

\node[draw=black,right, circle, fill=white, minimum size = 30*sqrt(50)-1pt, anchor=south,line width=1pt, scale=0.2] at (axis cs:2500,62.3) (n1){};
\node[draw=black,right, circle, fill=black!50!white, minimum size = 30*sqrt(5)-1pt,anchor=south,line width=1pt, scale=0.2] at (axis cs:2500,62.3) (n2) {};
\node[draw=black,right, circle, fill=white, minimum size = 30*sqrt(0.5)-1pt,  anchor=south, thin, scale=0.2] at (axis cs:2500,62.3) (n4) {};

\draw [-] (n1.north) -- ++ (1cm,0cm);
\draw [-] (n2.north) -- ++ (1cm,0cm);
\draw [-] (n1.north) -- ++ (1cm,0cm);
\draw [-] (n4.north) -- ++ (1cm,0cm);

\node[draw=none,anchor=north, right=of n1.north] (x) {\footnotesize 5$\times 10^7$};
\node[draw=none,anchor=west, right=of n2.north] (x) {\footnotesize 5$\times 10^6$};
\node[draw=none,anchor=west, right=of n4.north] (x) {\footnotesize 5$\times 10^5$};
\node[draw=none,anchor=center, below=of n4, yshift=1cm] (x) { \footnotesize \;\;\;\shortstack{model size \\ \# param.}};

\draw [thick, <->] (pointtrans) --node[midway, above]{\footnotesize \THROUGHPUTVSPTVTHREE} (pointtrans -| ezsp);

\draw [thick, <->] (spt) --node[midway, above]{\footnotesize \THROUGHPUTVSSPT} (spt -| ezsp);
     
\end{semilogxaxis}
\end{tikzpicture}

%% file: sections/2_related.tex
This section presents an overview of 3D deep learning with a focus on superpoint-based approaches.

\begin{figure*}[t]
\centering
\input{figures/pipeline}

\vspace{-2mm}
\caption{\textbf{\MODELNAME.}
A $60$k-parameter backbone embeds every point of the input scene into a low-dimensional space where adjacent points from different semantic classes (inter-edges) are pushed apart.  
A GPU-accelerated algorithm then clusters neighbouring points with similar embeddings, producing semantically homogeneous superpoints.  
Finally, a lightweight ($330$k-parameter) superpoint-level network assigns a label to each superpoint, which is broadcast back to its points for dense segmentation.}
\label{fig:pipeline}
\end{figure*}
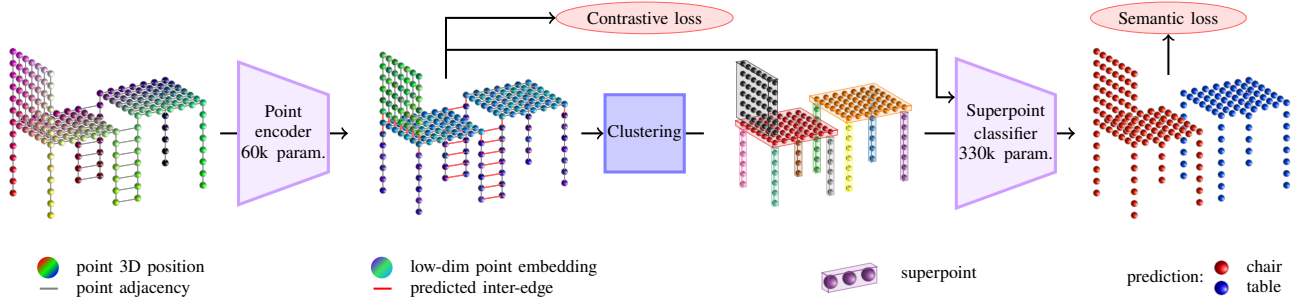

\para{3D Semantic Segmentation.}
Most existing approaches operate directly on raw points or on fine voxel grids, leveraging convolutional architectures~\cite{choy20194d,thomas2019kpconv,graham20183d} or, more recently, transformers~\cite{hui2021pptnet,Guo_2021,zhao2021point,wu2022point,wu2024point}.  
These methods achieve state-of-the-art accuracy, but at substantial computational cost.  
Models typically contain tens of millions of parameters and require significant memory and processing time, limiting their applicability to large-scale point clouds and real-time robotic systems.

\para{Superpoint-Based Semantic Segmentation.}
Partitioning a scene into \emph{superpoints}---spatially contiguous, geometrically homogeneous point clusters---drastically reduces computational load by shifting semantic reasoning from points to regions. The \emph{Superpoint Graph} (SPG)~\cite{landrieu2018spg} pioneered this approach, achieving competitive results with far fewer resources. The \emph{Superpoint Transformer} (SPT)~\cite{robert2023efficient} advanced the idea with hierarchical partitions and sparse self-attention, reaching SOTA accuracy with only $\sim$200k parameters. Extensions include instance~\cite{liang2021instance,sun2023superpoint}, 
panoptic~\cite{robert2024scalable}, weakly supervised~\cite{liu2021one}, and self-supervised~\cite{zhang2023growsp, nunes2022segcontrastd, liu2023udsunsupervised} segmentation. However, all rely on CPU-based, handcrafted partitions, which limit scalability and accuracy.

\para{Classical Superpoint Partitioning.}
Early 3D over-segmentation methods often used supervoxels, \eg, VCCS~\cite{papon2013voxel} via voxel-based $k$-means. 
Such methods require \emph{predefining the number of clusters} and are sensitive to initialization, preventing flexibility to scene size and local complexity. 
Saliency-guided~\cite{gao2017saliency} and density-adaptive~\cite{lin2018toward} variants improved boundaries, while graph-based optimization~\cite{benshabat2017pclv, guinard2017weakly, liang2021instance, han2020occuseg, thyagharajan2022segment} further improved adaptability. 
Nevertheless, these methods remain CPU-bound and largely depend on handcrafted features and hyperparameters. 
Sparsification methods~\cite{jin2024multiway, yew2022regtr,zhu2024spgroup3d} lower the number of points but do not produce valid partitions for segmentation.
Image-based methods~\cite{yang2023sam3d} typically project SLIC~\cite{achanta2012slic} or SAM~\cite{kirillov2023segmentanything} superpixel partitions onto 3D points, but require co-registered 2D images, and rely on 2D textures rather than 3D geometry.

\para{Learning Superpoint Partitions.}
Learning partitions is challenging due to the non-differentiability of standard clustering. GPU-friendly $k$-means~\cite{cheng2020cascaded, hui2021superpoint, zhang2023growsp, lu20253dlst} and region-offset~\cite{kang2023region} approaches enable differentiability but inherit $k$-means’ limitations and often require heavy encoders. GraphCut~\cite{hui2022learning} introduces a graph-based heuristic on learned embeddings for instance segmentation, while SAM-Graph~\cite{guo2024sam} transfers priors from large image models to generate superpoints for detection.  
Supervised SuperPoints (SSP)~\cite{landrieu2019point} directly train embeddings to highlight boundaries but still rely on CPU-based Cut Pursuit~\cite{landrieu2017cut}. Our method replaces the CPU-based solver with a massively parallel approximate algorithm, uses an edge-based surrogate loss, and employs a \emph{small} ($330$k-parameter) encoder. In practice, removing the CPU bottleneck and avoiding fixed cluster counts enable fast, scalable training and inference.

%% file: figures/pipeline.tex
\resizebox{\linewidth}{!}{

\def\xinput{0}
\def\xpointencoder{2.75}
\def\xedge{5}
\def\xlossedge{7.75}
\def\xcluster{7.75}
\def\xsp{9.75}
\def\xspt{12.75}
\def\xpredict{15}
\def\ymain{0}
\def\ylegend{-2}
\def\yloss{1.6}

\def\winput{3.5}
\def\wnetwork{1.0}
\definecolor{MODELCOLOR}{RGB}{207, 159, 255}
\definecolor{LOSSCOLOR}{RGB}{255, 100, 100}
\tikzset{
  stagebox/.style={
    draw=violet, thick,
    fill=violet!10,
    rounded corners=2pt,
    drop shadow={shadow scale=1, opacity=.25}
  }
}

\begin{tikzpicture}
\path[use as bounding box] (0.7, -2.5) rectangle (17, 2.1);%
\node [draw=none] (input) at (\xinput, \ymain) {
\resizebox{\winput cm}{!}{\input{figures/table_perspective}}
};

\node at (\xinput+0.5,\ylegend) {%
\begin{minipage}{\winput cm}
\centering 
\scriptsize
\begin{tabular}{r@{\quad}l}
  \begin{tikzpicture}[baseline=-.5ex]
  \shade[left color=red, right color=blue, middle color=green, shading angle=45] (0,0) circle (0.15cm);
  \end{tikzpicture}%
  &
  point 3D position
  \\
  \tikz[baseline=-.5ex]{\draw [thick, draw=black!50] (0,0) -- (0.25,0);}
    &
  point adjacency
  \end{tabular}
  \end{minipage}
};
\node [draw=MODELCOLOR, fill=MODELCOLOR!20, very thick, trapezium, text width=1cm, minimum height= \wnetwork cm, rotate=-90,trapezium angle=70,text centered, inner sep=1pt] (pointencoder) at (\xpointencoder,\ymain)
{\rotatebox{+90}{\scriptsize \shortstack{Point\\encoder\\ $60$k param.}}};
\node [draw=none] (edge) at (\xedge, \ymain) { \resizebox{\winput cm}{!}{\input{figures/table_embeddings}}};

\node [ellipse, draw=LOSSCOLOR, fill=LOSSCOLOR!20, inner sep=2pt] (lossedge) at (\xlossedge, \yloss) {\scriptsize Contrastive loss};

\node at (\xedge+0.5,\ylegend) {%
\begin{minipage}{\winput cm}
\centering
\scriptsize
\begin{tabular}{r@{\quad}l}
  \begin{tikzpicture}[baseline=-.5ex]
  \shade[left color=clrLeg, right color=clrFlat, middle color=clrVert, shading angle=45] (0,0) circle (0.15cm);
  \end{tikzpicture}%
  &
  low-dim point embedding
  \\
  \tikz[baseline=-.5ex]{\draw [thick, draw=black!50, red] (0,0) -- (0.25,0);}
    &
  predicted inter-edge
  \end{tabular}
  \end{minipage}
};
 \node [draw=blue!50, fill=blue!20, very thick, trapezium, text width=1cm, minimum height= \wnetwork cm, rotate=-90,trapezium angle=0,text centered, inner sep=1pt] (clustering) at (\xcluster,\ymain)
{\rotatebox{+90}{\scriptsize \shortstack{Clustering}}};
\node [draw=none] (super) at (\xsp, \ymain) { \resizebox{\winput cm}{!}{\input{figures/table_superpoint}}};

\node at (\xsp-0.25,\ylegend) {%
  \begin{minipage}{\winput cm}
    \centering
    \scriptsize
    \begin{tabular}{r@{\quad}l}
      \input{figures/miniexample_superpoint} &
      superpoint
    \end{tabular}
  \end{minipage}%
};
\node [draw=MODELCOLOR, fill=MODELCOLOR!20, very thick, trapezium, text width=1cm, minimum height= \wnetwork cm, rotate=-90,trapezium angle=70,text centered, inner sep=1pt] (spointencoder) at (\xspt,\ymain)
{\rotatebox{+90}{\scriptsize \shortstack{Superpoint \\ classifier \\ $330$k param.}}};
\node [draw=none] (prediction) at (\xpredict, \ymain) { \resizebox{\winput cm}{!}{\input{figures/table_semantic}}};

\node [ellipse, draw=LOSSCOLOR, fill=LOSSCOLOR!20, inner sep=2pt] (losspred) at (\xpredict, \yloss) {\scriptsize Semantic loss};

\node at (\xpredict+0.5,\ylegend) {%
\begin{minipage}{\winput cm}
\tikzset{
  point/.style={
    circle,
    ball color=blue!70!black,
    shading angle=210,
    inner sep= 2 pt,
    postaction={draw=white!80!black,line width=.25pt,opacity=.5} 
  }
}
\centering
\scriptsize
\begin{tabular}{l@{\;\;}r@{\quad}l}
\multirow{2}{*}{prediction:}
&
  \begin{tikzpicture}[baseline=-.5ex]
  \node[point, ball color=red] at (0, 0) {};
  \end{tikzpicture}%
  &
  chair
  \\
&\begin{tikzpicture}[baseline=-.5ex]
  \node[point, ball color=blue] at (0, 0) {};
  \end{tikzpicture}%
  &
  table
  \end{tabular}
  \end{minipage}
};
\draw [thick, ->] 
            (input) -- (pointencoder.south)
            (pointencoder.north) -- ++ (0.25,0);
\draw [thick, ->] 
            (edge) -- (clustering);
\draw [thick, ->]                  
            (clustering.north)  -- ++ (0.25,0)              
            (super) -- (spointencoder.south) (spointencoder.north) -- ++ (0.25,0);

\draw [thick, ->] (edge.north) ++ (0,-0.5) |- (lossedge);

\draw [thick, ->] 
(spointencoder.south) ++ (-0.25, +0.5) coordinate (tmp4)
(edge.north) ++ (0,-0.1) -| (tmp4) -- ++  (0.25,0);



\draw [thick, ->] (prediction.north) ++ (0,-0.5) -- (losspred);


\end{tikzpicture}

}

%% file: figures/table_perspective.tex
\tikzset{
  point/.style={
    circle,                     
    ball color=blue!70!black,   
    shading angle=210,           
    inner sep= \pointsize pt,              
    postaction={draw=white!80!black,line width=.5pt,opacity=.5} 
  }
}

\tikzset{
  edge/.style={
    line width=0.7pt,
    draw=black!70,
    line cap=round,
    postaction={
      draw=white,
      opacity=0.4,          
      line width=1pt,
      line cap=round        
    }
  }
}

\tikzset{
  interedge/.style={
  }
}

\newcommand{\bbox}[2]{%
}

\newcommand{\bboxc}[7]{%
}

\newcommand{\pcnode}[7]{%
  \pgfmathsetmacro{\r}{(#2-\xtable)/(-\xtable+1)}              
  \pgfmathsetmacro{\g}{#3}              
  \pgfmathsetmacro{\b}{(#4+1)/2}        
  \node[point,
        shade,shading=ball,
        ball color={rgb,1:red,\r;green,\g;blue,\b}
       ] (#1) at (#2,#3,#4) {};
}

\input{figures/table}

%% file: figures/table.tex
\def\npoint{6} 
\def\pointsize{3.7}
\pgfmathsetmacro{\step}{1/(\npoint)}
\def\xtable{-1.5}
\pgfmathsetmacro{\gapx}{\step}          
\pgfmathsetmacro{\tblZ}{0.0+\step}        
\pgfmathtruncatemacro{\nTblLeg}{\npoint + int(round(\tblZ*\npoint))}

\tikzset{
 xedge/.style={
shorten <= -\pointsize pt
}}

\tikzset{
 zedge/.style={
shorten >=-0.5*\pointsize pt
}}

\tikzset{
 yedge/.style={
shorten <=-0.5*\pointsize pt
}}

\def\clrLeg {red,0.36;green,0.18;blue,0.85}  
\def\clrFlat{red,0.05;green,0.70;blue,0.95}  
\def\clrVert{red,0.25;green,0.90;blue,0.35}  

  \tdplotsetmaincoords{70}{150}

  \begin{tikzpicture}[tdplot_main_coords,scale=4]

\foreach \L/\dx/\dy in {0/0/0,1/1/0,2/0/1,3/1/1}{
  \pgfmathsetmacro{\cx}{\xtable+\gapx+\dx}    
  \pgfmathsetmacro{\cy}{\dy}            
  \foreach \k in {1,...,\nTblLeg}{
    \pgfmathsetmacro{\z}{\tblZ - \k/\npoint}
    \pcnode{tleg_\L_\k}{\cx}{\cy}{\z}{table}{leg}{\L+1}
  }
    \bboxc{\cx-\step/2}{\cx+\step/2}
      {\cy-\step/2}{\cy+\step/2}
      {-1}{\tblZ}{\L+1}
}

\foreach \L in {0,...,3}{
  \foreach \k in {1,...,\nTblLeg}{
    \ifnum\k<\nTblLeg
      \pgfmathtruncatemacro{\kp}{\k+1}
      \draw[edge, zedge] (tleg_\L_\k) -- (tleg_\L_\kp);
    \fi
  }
}

\foreach \i in {0,...,\npoint}{
  \foreach \j in {0,...,\npoint}{
    \pgfmathsetmacro{\x}{\xtable+\gapx + \i/\npoint}
    \pgfmathsetmacro{\y}{\j/\npoint}
    \pcnode{tbl_\i_\j}{\x}{\y}{\tblZ}{table}{flat}{5}
  }
}

 \bbox{\xtable+\gapx-\step/2}{\xtable+\gapx+1+\step/2} 
     {-\step/2}{1+\step/2}                            
     {\tblZ-\step/4}{\tblZ+\step/4}                   
     {5}   

\draw[edge, zedge, interedge] (tbl_0_0)         -- (tleg_0_1);
\draw[edge, zedge, interedge] (tbl_\npoint_0)   -- (tleg_1_1);
\draw[edge, zedge, interedge] (tbl_0_\npoint)   -- (tleg_2_1);
\draw[edge, zedge, interedge] (tbl_\npoint_\npoint) -- (tleg_3_1);

\foreach \i in {0,...,2}{
  \foreach \j in {0,...,2}{
    \pgfmathsetmacro{\x}{\xtable+\gapx + \i/\npoint}
    \pgfmathsetmacro{\y}{\j/\npoint}
    \pcnode{tbl_\i_\j}{\x}{\y}{\tblZ}{table}{flat}{5}
  }
}

\foreach \i in {0,...,\npoint}{
  \foreach \j in {0,...,\npoint}{
    \ifnum\i<\npoint
      \pgfmathtruncatemacro{\ip}{\i+1}
      \draw[edge, xedge] (tbl_\i_\j) -- (tbl_\ip_\j);
    \fi
    \ifnum\j<\npoint
      \pgfmathtruncatemacro{\jp}{\j+1}
      \draw[edge, yedge] (tbl_\i_\j) -- (tbl_\i_\jp);
    \fi
  }
}

\foreach \l/\cx/\cy in {0/0/0,1/1/0,2/0/1,3/1/1}{
  \foreach \k in {1,...,\npoint}{
    \pgfmathsetmacro{\z}{-\k/\npoint}

    \pcnode{leg_\l_\k}{\cx}{\cy}{\z}{chair}{leg}{6+\l}
  }
  \bboxc{\cx-\step/2}{\cx+\step/2}   
     {\cy-\step/2}{\cy+\step/2}   
     {-1}{0.0}                  
     {6+\l} 
}

\foreach \L in {0,...,3}{               
  \foreach \k in {1,...,\npoint}{
    \ifnum\k<\npoint
      \pgfmathtruncatemacro{\kp}{\k+1}
      \draw[edge, zedge] (leg_\L_\k) -- (leg_\L_\kp);
    \fi
  }
}

\foreach \j in {1,...,\npoint}{
\pgfmathsetmacro{\jp}{\j+1}
\draw[edge, shorten >= 2*\pointsize pt, interedge] (leg_2_\j) -- (tleg_3_\jp);
\draw[edge, shorten >= 2*\pointsize pt, interedge]  (leg_0_\j) --  (tleg_1_\jp);
}

\foreach \i in {0,...,\npoint}{
  \foreach \j in {0,...,\npoint}{
    \pgfmathsetmacro{\x}{\i/\npoint}
    \pgfmathsetmacro{\y}{\j/\npoint}
    \pcnode{seat_\i_\j}{\x}{\y}{0}{chair}{flat}{10}
  }
}
\bbox{-\step/2}{1+\step/2}    
     {-\step/2}{1+\step/2}    
     {-\step/4}{\step/4}      
     {10}

\foreach \i in {0,...,\npoint}{
  \foreach \j in {0,...,\npoint}{
    \ifnum\i<\npoint
      \pgfmathtruncatemacro{\ip}{\i+1}
      \draw[edge, xedge] (seat_\i_\j) -- (seat_\ip_\j);
    \fi
    \ifnum\j<\npoint
      \pgfmathtruncatemacro{\jp}{\j+1}
      \draw[edge, yedge] (seat_\i_\j) -- (seat_\i_\jp);
    \fi
  }
}

\foreach \j in {0,...,\npoint}{
  \foreach \k in {1,...,\npoint}{
    \pgfmathsetmacro{\y}{\j/\npoint}
    \pgfmathsetmacro{\z}{\k/\npoint}
    \pcnode{back_\j_\k}{1}{\y}{\z}{chair}{vertflat}{11}
  }
}
\bbox{1}{1+\step}             
     {+\step/2}{1+\step/2}                   
     {2*\step/2}{1+\step/2}           
     {11}

\foreach \j in {0,...,\npoint}{
  \foreach \k in {1,...,\npoint}{
    \ifnum\j<\npoint
      \pgfmathtruncatemacro{\jp}{\j+1}
      \draw[edge, yedge] (back_\j_\k) -- (back_\jp_\k);
    \fi
    \ifnum\k<\npoint
      \pgfmathtruncatemacro{\kp}{\k+1}
      \draw[edge, zedge] (back_\j_\kp) -- (back_\j_\k);
    \fi
  }
}

\foreach \j in {0,...,\npoint}{
  \draw[edge, zedge, interedge] (seat_\npoint_\j) -- (back_\j_1);
}

\draw[edge, zedge, interedge] (seat_0_0)               -- (leg_0_1);
\draw[edge, zedge, interedge] (seat_\npoint_0)         -- (leg_1_1);
\draw[edge, zedge, interedge] (seat_0_\npoint)         -- (leg_2_1);
\draw[edge, zedge, interedge] (seat_\npoint_\npoint)   -- (leg_3_1);

\foreach \i in {0,...,2}{
  \foreach \j in {0,...,2}{
    \pgfmathsetmacro{\x}{\i/\npoint}
    \pgfmathsetmacro{\y}{\j/\npoint}
    \pcnode{seat_\i_\j}{\x}{\y}{0}{chair}{flat}{10}
  }
}
\pgfmathtruncatemacro{\kZero}{\npoint + int(round(\tblZ*\npoint)) - \npoint} 
\foreach \j in {0,...,\npoint}{
  \pgfmathsetmacro{\y}{\j/\npoint}
}

\draw[edge, interedge]  (seat_0_\npoint) --  (tleg_3_1);
\draw[edge, interedge]  (seat_0_0) --  (tleg_1_1);

  \end{tikzpicture}

%% file: figures/table_embeddings.tex
\tikzset{
  point/.style={
    circle,                     
    ball color=blue!70!black,   
    shading angle=210,           
    inner sep= \pointsize pt,              
    postaction={draw=white!80!black,line width=.5pt,opacity=.5} 
  }
}

\tikzset{
  edge/.style={
    line width=0.7pt,
    draw=black!70,
    line cap=round,
    postaction={
      draw=white,
      opacity=0.4,          
      line width=1pt,
      line cap=round        
    }
  }
}

\tikzset{
  interedge/.style={
  draw=red,
  line width=2pt,
  }
}

\tikzset{
  bboxface/.style={
    line width=0.3pt,
    fill opacity=0.15,
    draw opacity=0.6
  }
}
\newcommand{\bbox}[7]{
}

\newcommand{\bboxc}[7]{%
}

\def\clrLeg {red,0.36;green,0.18;blue,0.85}  
\def\clrFlat{red,0.05;green,0.70;blue,0.95}  
\def\clrVert{red,0.25;green,0.90;blue,0.35}  

\def\noise{0.20}

\newcommand{\pcnode}[7]{%
  \ifstrequal{#6}{leg}{
      \def\rb{0.36}\def\gb{0.18}\def\bb{0.85}%
    }{%
      \ifstrequal{#6}{flat}{
          \def\rb{0.05}\def\gb{0.70}\def\bb{0.95}%
        }{
          \def\rb{0.25}\def\gb{0.90}\def\bb{0.35}%
        }%
    }%
  \pgfmathsetmacro{\dr}{\noise*(rand-0.5)}
  \pgfmathsetmacro{\r}{max(min(\rb+\dr,1),0)}
  
  \pgfmathsetmacro{\dg}{\noise*(rand-0.5)}
  \pgfmathsetmacro{\g}{max(min(\gb+\dg,1),0)}
  
  \pgfmathsetmacro{\db}{\noise*(rand-0.5)}
  \pgfmathsetmacro{\b}{max(min(\bb+\db,1),0)}
  \node[
    point,
    shade, shading=ball,
    ball color={rgb,1:red,\r;green,\g;blue,\b}
  ] (#1) at (#2,#3,#4) {};
}

\input{figures/table}

%% file: figures/table_superpoint.tex
\tikzset{
  point/.style={
    circle,                     
    ball color=blue!70!black,   
    shading angle=210,           
    inner sep= \pointsize pt,              
    postaction={draw=white!80!black,line width=.5pt,opacity=.5} 
  }
}

\tikzset{
  edge/.style={
  opacity=0.0,
    line width=0.7pt,
    draw=black!70,
    line cap=round,
    postaction={
      draw=white,
      opacity=0.0,          
      line width=1pt,
      line cap=round        
    }
  }
}
\definecolor{pcA}{RGB}{0, 0, 0}   
\definecolor{pcB}{RGB}{55 ,126,184}   
\definecolor{pcC}{RGB}{77 ,175, 74}   
\definecolor{pcD}{RGB}{152, 78,163}   
\definecolor{pcE}{RGB}{255,255, 51}   
\definecolor{pcF}{RGB}{255,127,  0}   
\definecolor{pcG}{RGB}{166, 86, 40}   
\definecolor{pcH}{RGB}{247,129,191}   
\definecolor{pcI}{RGB}{153,153,153}   
\definecolor{pcJ}{RGB}{102,194,165}   
\definecolor{pcK}{RGB}{228, 26, 28}   
\definecolor{pcL}{RGB}{50 , 50, 50}   
\definecolor{pcM}{RGB}{252,141, 98}   

\newcommand{\pcpick}[1]{%
  \ifcase#1 pcA\or pcB\or pcC\or pcD\or pcE\or pcF%
           \or pcG\or pcH\or pcI\or pcJ\or pcK\or pcL\or pcM\fi}

\tikzset{
  bboxface/.style={
    line width=0.3pt,
    fill opacity=0.15,
    draw opacity=0.6
  }
}

\newcommand{\bbox}[7]{
  \pgfmathtruncatemacro{\pcidx}{#7}%
  \edef\pcColor{\pcpick{\pcidx}}%

  \filldraw[bboxface, fill=\pcColor, draw=\pcColor]
    (#2,#3,#5) -- (#2,#4,#5) -- (#2,#4,#6) -- (#2,#3,#6) -- cycle;

  \filldraw[bboxface, fill=\pcColor, draw=\pcColor]
    (#1,#4,#5) -- (#2,#4,#5) -- (#2,#4,#6) -- (#1,#4,#6) -- cycle;

  \filldraw[bboxface, fill=\pcColor, draw=\pcColor]
    (#1,#3,#6) -- (#2,#3,#6) -- (#2,#4,#6) -- (#1,#4,#6) -- cycle;
}

\newcommand{\bboxc}[7]{%
  \pgfmathsetmacro{\padxy}{\step/4}%
  \pgfmathsetmacro{\padz}{\step/4}%
  \bbox{#1+\padxy}{#2-\padxy}%
       {#3+\padxy}{#4-\padxy}%
       {#5-\padz}{#6-2*\padz}%
       {#7}%
}


\tikzset{
  interedge/.style={
  }
}

\newcommand{\pcnode}[7]{%
  \pgfmathtruncatemacro{\pcidx}{#7}
  \edef\pcColor{\pcpick{\pcidx}}%
  \node[
    point,
    shade, shading=ball,
    ball color=\pcColor
  ] (#1) at (#2,#3,#4) {};
}

\input{figures/table}

%% file: figures/miniexample_superpoint.tex
\def\pointsize{2}
\def\npoint{4} 
\pgfmathsetmacro{\step}{1/(\npoint)}

\tikzset{
  point/.style={
    circle,                     
    ball color=blue!70!black,   
    shading angle=210,           
    inner sep= \pointsize pt,
    postaction={draw=white!80!black,line width=.25pt,opacity=.5} 
  }
}

\definecolor{pcA}{RGB}{0, 0, 0}   
\definecolor{pcB}{RGB}{55 ,126,184}   
\definecolor{pcC}{RGB}{77 ,175, 74}   
\definecolor{pcD}{RGB}{152, 78,163}   
\definecolor{pcE}{RGB}{255,255, 51}   
\definecolor{pcF}{RGB}{255,127,  0}   
\definecolor{pcG}{RGB}{166, 86, 40}   
\definecolor{pcH}{RGB}{247,129,191}   
\definecolor{pcI}{RGB}{153,153,153}   
\definecolor{pcJ}{RGB}{102,194,165}   
\definecolor{pcK}{RGB}{228, 26, 28}   
\definecolor{pcL}{RGB}{50 , 50, 50}   
\definecolor{pcM}{RGB}{252,141, 98}   

\newcommand{\pcpick}[1]{%
  \ifcase#1 pcA\or pcB\or pcC\or pcD\or pcE\or pcF%
           \or pcG\or pcH\or pcI\or pcJ\or pcK\or pcL\or pcM\fi}

\tikzset{
  bboxface/.style={
    line width=0.3pt,
    fill opacity=0.15,
    draw opacity=0.6
  }
}

\newcommand{\bbox}[7]{
  \pgfmathtruncatemacro{\pcidx}{#7}%
  \edef\pcColor{\pcpick{\pcidx}}%

  \filldraw[bboxface, fill=\pcColor, draw=\pcColor]
    (#2,#3,#5) -- (#2,#4,#5) -- (#2,#4,#6) -- (#2,#3,#6) -- cycle;

  \filldraw[bboxface, fill=\pcColor, draw=\pcColor]
    (#1,#4,#5) -- (#2,#4,#5) -- (#2,#4,#6) -- (#1,#4,#6) -- cycle;

  \filldraw[bboxface, fill=\pcColor, draw=\pcColor]
    (#1,#3,#6) -- (#2,#3,#6) -- (#2,#4,#6) -- (#1,#4,#6) -- cycle;
}

\newcommand{\bboxc}[7]{%
  \pgfmathsetmacro{\padxy}{\step/4}%
  \pgfmathsetmacro{\padz}{\step/4}%
  \bbox{#1+\padxy}{#2-\padxy}%
       {#3+\padxy}{#4-\padxy}%
       {#5-\padz}{#6-2*\padz}%
       {#7}%
}


\tikzset{
  interedge/.style={
  }
}

\newcommand{\pcnode}[7]{%
  \pgfmathtruncatemacro{\pcidx}{#7}
  \edef\pcColor{\pcpick{\pcidx}}%
  \node[
    point,
    shade, shading=ball,
    ball color=\pcColor
  ] (#1) at (#2,#3,#4) {};
}

  \tdplotsetmaincoords{70}{150}

  \begin{tikzpicture}[tdplot_main_coords,scale=1,baseline=-.5ex]
        \foreach \x/\idx in {0/3, 0.3/3, 0.6/3}{
          \node[point, ball color=\pcpick{\idx}] at (\x,0) {};
        }
        \bboxc{-0.15}{0.75}{-0.15}{0.15}{-0.05}{0.25}{3}
  \end{tikzpicture}

%% file: figures/table_semantic.tex
\tikzset{
  point/.style={
    circle,                     
    ball color=blue!70!black,   
    shading angle=210,           
    inner sep= \pointsize pt,              
    postaction={draw=white!80!black,line width=.5pt,opacity=.5} 
  }
}

\tikzset{
  edge/.style={
    opacity=0.0,
    line width=0.0pt,
    draw=black!70,
    line cap=round,
    postaction={
      draw=white,
      opacity=0.0,          
      line width=0pt,
      line cap=round        
    }
  }
}

\tikzset{
  interedge/.style={
  }
}

\newcommand{\bbox}[7]{%
}

\newcommand{\bboxc}[7]{%
}

\newcommand{\pcnode}[7]{
  \ifstrequal{#5}{chair}{%
    \def\pcColor{red!80!orange}
  }{%
    \def\pcColor{blue!70!cyan}
  }%
  \node[
    point,
    shade, shading=ball,
    ball color=\pcColor
  ] (#1) at (#2,#3,#4) {};
}

\input{figures/table}

%% file: sections/3_method2.tex
We consider a point cloud $\Cloud$ composed of points with $F$ features: spatial coordinates and potential radiometric attributes (color, intensity). Each point $p$ has a semantic label $\class(p) \in [1,C]$, with $C$ classes. Our goal is to efficiently predict the semantic labels of all points. We first learn low-dimensional embeddings tailored for detecting semantic transitions (\cref{subsec:transition}), then propose an efficient GPU-based superpoint partitioning method (\cref{subsec:partition}), and finally show how our approach can be interfaced with a superpoint classification model for fast end-to-end semantic segmentation (\cref{subsec:classification}). 
See \cref{fig:pipeline} for a visual overview.

\subsection{Detecting Semantic Transition}
\label{subsec:transition}

We train a lightweight network to compute point embeddings optimized for detecting semantic transitions.

\para{Motivation.} 
Directly embedding points in a semantic space is typically challenging and requires large networks with extensive receptive fields. Instead, we exploit the simpler insight that semantic boundaries usually correspond to sharp contrasts in geometry or radiometry, such as the geometric discontinuity between a chair and the floor, or the change of color between doors and walls. Detecting these \emph{semantic transitions} is therefore a much simpler problem than semantic segmentation, does not require global information, and should be achievable with a simpler model.

\para{Point Embedding.} 
We define the embedding function
$
    \phi^{\text{point}}: \mathbb{R}^{\vert\Cloud\vert \times F} \mapsto \mathbb{R}^{\vert\Cloud\vert \times M},
$
which associates each point $p$ with an $M$-dimensional embedding vector. We denote by $\Feat = \phi^{\text{point}}(\Cloud)$ the embedding matrix.

\para{Semantic Transition Prediction.}
We aim to learn embeddings that remain homogeneous within objects while being sharply contrasted across semantic boundaries.  
Following Robert \etal~\cite{robert2024scalable}, we formulate transition detection as a binary edge classification task.  
We first define the pairwise affinity between two points $p$ and $q$ as:
\begin{align}\label{eq:softmax}
    a_{p,q} = \exp\left(-{\lVert \feat_p - \feat_q\rVert}/{\tau}\right)~,    
\end{align}
with $\tau>0$ a temperature parameter. 
We construct an undirected graph $(\Cloud,\Edges)$ connecting points with their $k$ nearest neighbours.  
We define \emph{intra-edges}, $\Edges_\text{intra}=\{(p,q)\in\Edges\mid \class(p)=\class(q)\}$, and \emph{inter-edges}, $\Edges_\text{inter}=\{(p,q)\in\Edges \mid  \class(p)\neq\class(q)\}$.
We encourage $a_{p,q}\!\approx\!1$ for intra-edges and $a_{p,q}\!\approx\!0$ for inter-edges with a \emph{contrastive loss}:
\begin{align} \label{eq:loss}
    \cL &= 
    \sum_{(p,q)\in\Edges_\text{intra}}\!\!-\log(a_{p,q})
    +
    \sum_{(p,q)\in\Edges_\text{inter}}\!\!-\log(1 - a_{p,q})~.
\end{align}
To improve diversity in the learned embeddings, we apply adaptive sampling by randomly dropping intra-edges until they constitute at most $\rho_\text{intra}$ of the sampled edges.


\subsection{Fast Superpoint Partitioning}
\label{subsec:partition}
We now present a GPU-accelerated method to efficiently compute a superpoint partition from the learned embeddings.

\para{Motivation.} 
Clustering-based approaches such as $K$-means require a fixed cluster count, and thus struggle with variable scene size and complexity.  
We adopt a graph-based partition approach by minimizing the following contour-regularized energy~\cite{landrieu2017cut}:
\begin{gather}
\label{eq:gmpp}
\argmin_{\FeatY \in \mathbb{R}^{\vert \Cloud \vert \times M}}\Omega(\FeatY; \Feat, \Edges)\;\text{with} \\
\Omega(\FeatY; \Feat,\Edges)=
\sum_{p\in\Cloud}\!
    \|\Feat_p - \FeatY_p\|^2
+\!
\regul \!\!\! \sum_{(p,q)\in\Edges}\!\!\!
    \edgeweights_{p,q}
    \|\FeatY_p - \FeatY_q\|_0~, \notag
\end{gather}
\noindent where $\| x \|_0=0$ if $x=0$ or otherwise $1$; $\edgeweights_{p,q}>0$ are edge weights; $\regul > 0$ the regularization strength. Minimizing this energy produces a piecewise-constant approximation $\FeatY$ of the embeddings $\Feat$, whose components form our superpoints. As $\Feat$ is trained to be homogeneous within objects and contrasted at their interface, the superpoints should be semantically coherent. Existing solvers for \cref{eq:gmpp} are typically CPU-bound~\cite{raguet2019parallel}, which can become computational bottlenecks.

\para{Combinatorial Clustering.}
We recast the non-continuous, non-convex optimization problem of \cref{eq:gmpp} as a combinatorial problem which we can efficiently approximate with a parallel greedy bottom-up merging strategy. 
Let $\Partition$ denote a partition of $\Cloud$ into superpoints: each superpoint $P \in \Partition$ defines a connected component of the graph $(\Cloud,\Edges)$ and $\cup_{P\in\Partition}=\Cloud$. We define the adjacency  between superpoints as follows:
\begin{align}
    \SEdges = \{(P,Q)\in\Partition^2\mid\exists(p,q)\in\Edges, p\in P, q\in Q\}~.
\end{align} 
We associate each superpoint $P$ with its mean embedding: 
\begin{align}
    \Feat_P = \frac{1}{|P|}\sum_{p \in P} \feat_p~.
\end{align}
We then define the point embedding matrix $\Feat^\Partition$ in $\mathbb{R}^{\vert \Cloud \vert \times M}$ which associates each point $p$ with the value $\Feat_P $ of the superpoint $P$ it belongs to:
\begin{align}
   \Feat^\Partition_p = \Feat_P\;\text{for}\;P\;\text{such that}\;p\in P~.
\end{align}
This allows us to restate \cref{eq:gmpp} as a combinatorial problem of minimizing $\Omega(\Partition) = \Omega(\Feat^\Partition; \Feat, \Edges)$ with respect to $\cP$. To do so, we use the following proposition:

\begin{prop}
Merging adjacent superpoints $(P,Q) \in \SEdges$ decreases $\Omega(\Partition)$ by the following \emph{edge merge gain:}
\begin{align}\label{eq:delta}
    \Delta(P,Q) = -\frac{|P||Q|}{|P|+|Q|}\|\Feat_P - \Feat_Q\|^2 + \regul\!\!\!\!\!\sum_{(p,q)\in(P \times Q)\cap\Edges}\!\!\!\!\!\!\!\!\!\!\!\edgeweights_{p,q}~.
\end{align}
\end{prop}

\begin{figure}[t]
    \centering
    \input{figures/merging}
    \vspace{-1mm}
    \caption{\textbf{Parallel Combinatorial Partition.} 
Our algorithm greedily approximates a graph signal with piecewise-constant components. 
Conflicting merges (nodes with multiple outgoing edges) are removed, enabling an efficient parallel implementation on GPUs.}
    \label{fig:merge}
\end{figure}
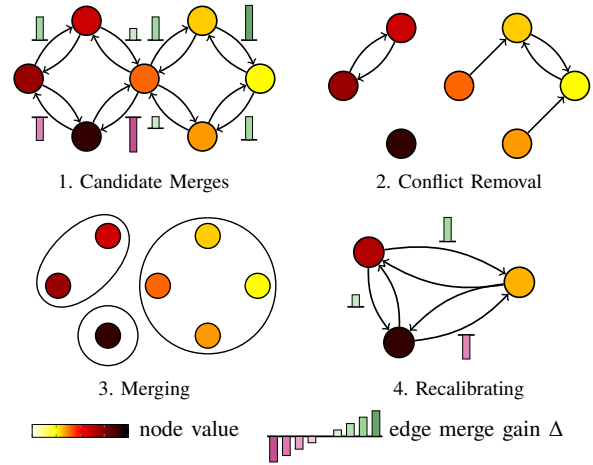

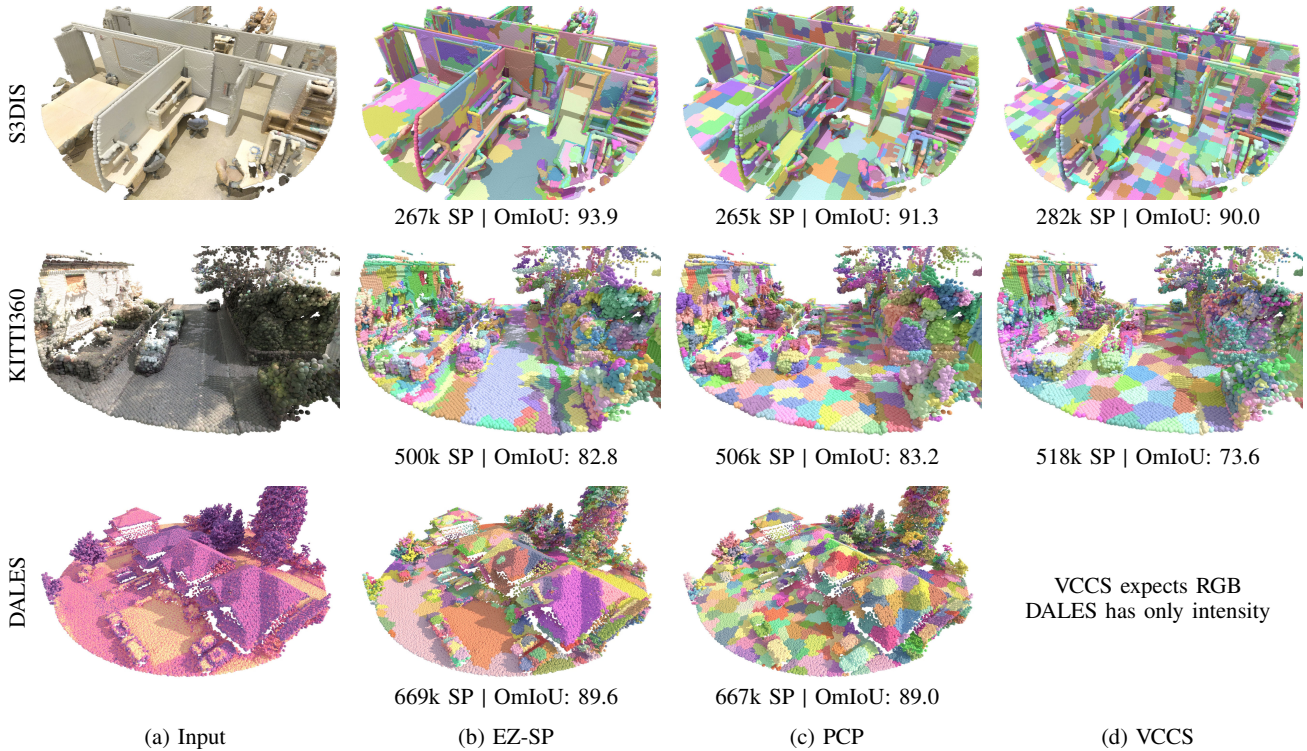
\begin{figure*}[t]
\centering
\input{figures/quali_over}

\vspace{-2mm}
\caption{
\textbf{Partition Examples.} 
Visualization of point cloud partitions across three datasets and three partitioning algorithms. 
\Cref{sec:data} shows the full dataset sizes; we also report, for each configuration, the resulting number of superpoints and the partition purity over the validation dataset (all folds for S3DIS).
}
\label{fig:quali_over}
\end{figure*}

\para{Parallel Implementation.}  
The combinatorial problem defined above can be approximated with a greedy merging strategy: at each step, adjacent superpoints $(P,Q)$ with the energy gain $\Delta(P,Q)$ are merged.
This process is inherently sequential, since each merge alters subsequent gains, making naive approaches ill-suited for GPUs.  
We therefore propose a bottom-up, GPU-parallel algorithm, illustrated in \cref{fig:merge}:  

\begin{compactitem}
    \item[0.] \textbf{Initialization.}  
    Each point is its own superpoint:  
    $\Partition = \{\{p\} \mid p \in \Cloud\}$ with adjacency $\SEdges = \Edges$.
    
    \item[1.] \textbf{Candidate Merges.}  
    We construct $\SEdges_\text{merge}$ the set of \emph{directed} edges $(P \rightarrow Q)$ for $(P,Q) \in \SEdges$ satisfying either  
    $\Delta(P,Q) > 0$ or $|P| < \sizemin$,  
    where $\sizemin$ is the minimum superpoint size.  
    If $\SEdges_\text{merge}$ is empty, return $\Partition$.  
    
    \item[2.] \textbf{Conflict Removal.}  
    To prevent conflicting merges, each superpoint may appear at most once as a source in $\SEdges_\text{merge}$.  
    For each $P$, we retain only the outgoing edge $(P \rightarrow \cdot)$ with the highest merge gain $\Delta$.  

    \item[3.] \textbf{Merging.}  
The remaining edges $\SEdges_\text{merge}$ define a directed merge graph $(\Partition,\SEdges_\text{merge})$.  
We compute its weakly connected components and update $\Partition$ with the resulting merged sets.  
This allows chain merges (\eg $(P\!\to\!R)$ and $(Q\!\to\!R)$) to be resolved in a single iteration.  

    \item[4.] \textbf{Recalibration.}  
    Update the node embeddings and merge gains for the new adjacency graph, then return to Step~1.
\end{compactitem}
Our approach makes heavy use of the highly optimized \texttt{scatter} operation~\cite{fey2019fast}, enabling efficient GPU parallelization.
Detailed pseudo-code and proofs of correctness will be released alongside an open-source GPU implementation.

\para{Hierarchical Partition.}  
The proposed algorithm can be applied recursively to produce a \emph{hierarchical} set of partitions, \ie  
$\Partition^{(1)}, \dots, \Partition^{(L)}$, where $\Partition^{(1)}$ is a partition of $\Cloud$ and  
$\Partition^{(l+1)}$ is a partition of $\Partition^{(l)}$.  
This is straightforward to implement by maintaining the adjacency graph between components at each stage.  
Such multi-scale partitioning is useful for downstream processing, as discussed in the next section.

\subsection{Semantic Segmentation}
\label{subsec:classification}

\begin{figure*}[t]
\centering
\input{figures/over}
\vspace{-2mm}
\caption{
\textbf{Oversegmentation Performance.} 
Oracle mIoU as a function of the number of superpoints on S3DIS, KITTI-360, and DALES. 
We also report the throughput (from raw points to superpoints) on S3DIS, with error bars indicating variance across configurations. 
\MODELNAME~achieves partition purity comparable to or better than PCP while being \emph{over \OVERSEGSPEEDUPVSPCP faster}.
}
\label{fig:over}
\end{figure*}
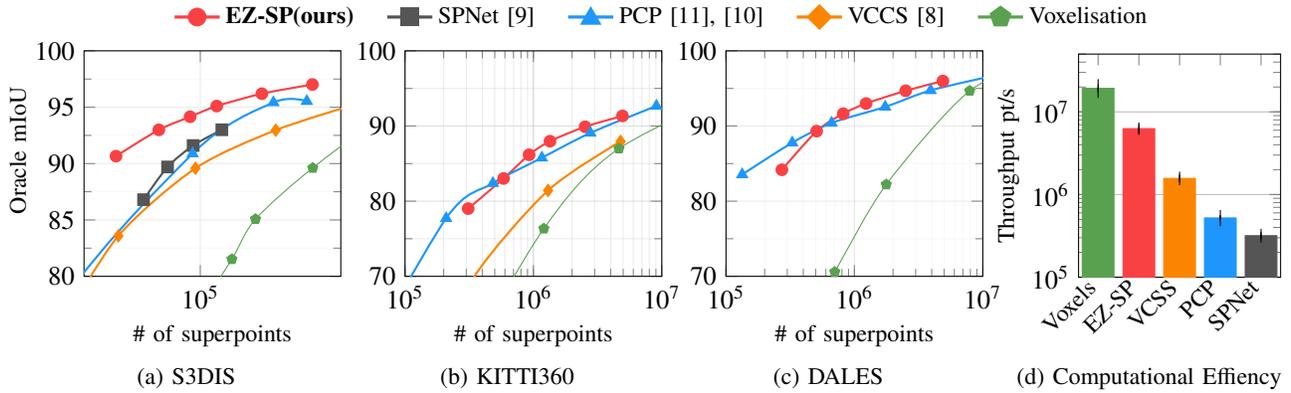

Once the initial point cloud $\Cloud$ is partitioned into superpoints $\Partition$,  
we can apply a superpoint-based classifier to predict their semantic labels.  
Labels are then broadcast from superpoints back to their constituent points,  
allowing the inference stage to operate entirely on the much smaller set $\Partition$  
while still producing dense predictions over $\Cloud$.  

\para{Superpoint Classification.}  
For classification, we employ the SuperPoint Transformer (SPT)~\cite{robert2023efficient} due to its strong balance between accuracy and efficiency. 
We retain the default configuration with three key modifications:  
\begin{compactitem}
    \item \textbf{Simplified Hyperparameters:} We remove all CPU-bound preprocessing and their  hyperparameters.  
    Partition coarseness is controlled solely by one parameter per partition level: the minimum superpoint size.  
    \item \textbf{Efficient GPU Pipeline:} because partitioning is fully GPU-based, vectorized operations run faster, and costly CPU--GPU data transfers are eliminated or optimized.
    \item \textbf{Hierarchical Architecture:} we preserve most of the original SPT design but extend it to three nested partition levels to leverage our hierarchical superpoints.  
\end{compactitem}

%% file: figures/merging.tex
\def\valA{0.3}
\def\valB{0.4}
\def\valC{0.1}
\def\valD{0.6}
\def\valE{0.65}
\def\valF{0.75}
\def\valG{0.7}

\resizebox{\linewidth}{!}{
\begin{tabular}{c@{\quad}c}
\captionsetup[subfigure]{labelformat=simple,labelsep=period,
                         labelfont=normalfont,textfont=normalfont}
\renewcommand\thesubfigure{\arabic{subfigure}}
\begin{subfigure}{0.5\linewidth}
\resizebox{\linewidth}{!}{
\begin{tikzpicture}[scale=1]
    \input{figures/graph}

    \draw [thick, bend left=20, ->] (n1) to node[midway, above left=-1mm and -1mm]
  {\BarGlyph{+2}} (n2);
    \draw [thick, bend left=-20, <-] (n1) to node[midway, below left= -1mm and -1mm] {\BarGlyph{-2}} (n3);
    \draw [thick, bend left=20, ->] (n2) to node[midway, above right= -1mm and -1mm] {\BarGlyph{+1}} (n4);
    \draw [thick, bend left=-20, <-] (n3) to node[midway, below right = -1mm and -1mm] {\BarGlyph{-3}} (n4);
    \draw [thick, bend left=-20, <-] (n4) to node[midway, below left= -1mm and -1mm] {\BarGlyph{+1}} (n5);
    \draw [thick, bend left=20, ->] (n4) to node[midway, above left= -1mm and -1mm] {\BarGlyph{+2}} (n7);
    \draw [thick, bend left=-20, <-] (n6) to node[midway, above right= -1mm and -1mm] {\BarGlyph{+3}} (n7);
    \draw [thick, bend left=20, ->] (n6) to node[midway, below right= -1mm and -1mm] {\BarGlyph{+2}} (n5);

    \draw [thick, bend left=-20, <-] (n1) to (n2);
    \draw [thick, bend left=+20, ->] (n1) to (n3);
    \draw [thick, bend left=-20, <-] (n2) to (n4);
    \draw [thick, bend left=+20, ->] (n3) to (n4);
    \draw [thick, bend left=+20, ->] (n4) to (n5);
    \draw [thick, bend left=-20, <-] (n4) to (n7);
    \draw [thick, bend left=+20, ->] (n6) to (n7);
    \draw [thick, bend left=-20, <-] (n6) to (n5);
\end{tikzpicture}
}
\caption{Candidate Merges}
\end{subfigure}
     &  
\captionsetup[subfigure]{labelformat=simple,labelsep=period,
 labelfont=normalfont,textfont=normalfont}
\renewcommand\thesubfigure{\arabic{subfigure}}
 \begin{subfigure}{0.5\linewidth}
 \resizebox{\linewidth}{!}{
\begin{tikzpicture}[scale=1]
    \input{figures/graph}

    \draw [->, thick, bend left=20] (n1) to (n2);
    \draw [->, thick, bend left=20] (n2) to (n1);
    \draw[->, thick, bend left=20]  (n6) to (n7);
    \draw[->, thick, bend left=20]  (n7) to (n6);
    \draw [thick, ->] (n4) -- (n7);   
    \draw [thick, <-] (n6) -- (n5);
    
\end{tikzpicture}
}
\caption{Conflict Removal}
\end{subfigure}   
\\[-2mm]
\captionsetup[subfigure]{labelformat=simple,labelsep=period,
                         labelfont=normalfont,textfont=normalfont}
\renewcommand\thesubfigure{\arabic{subfigure}}
 \begin{subfigure}{0.5\linewidth}
 \resizebox{\linewidth}{!}{
\begin{tikzpicture}[scale=1]

    \input{figures/graph}

    \draw [thick] ($(n1.south west) + (-0.1,-0.1)$) edge[out=-45,in=-45] ($(n2.north east) + (+0.1,+0.1)$);

    \draw [thick] ($(n1.south west) + (-0.1,-0.1)$) edge[out=135,in=135] ($(n2.north east) + (+0.1,+0.1)$);

    \draw[thick] (n3.center) circle (0.6);

    \draw [thick] ($(n4.west) + (-0.1,0)$) edge[out=-90,in=180] ($(n5.south) + (+0.0,-0.1)$);
    \draw [thick] ($(n5.south) + (+0.0,-0.1)$) edge[out=0,in=-90] ($(n6.east) + (+0.1,-0)$); 
    \draw [thick] ($(n6.east) + (+0.1,-0)$) edge[out=90,in=0] ($(n7.north) + (+0.0,+0.1)$);
    \draw [thick] ($(n7.north) + (+0.0,+0.1)$) edge[out=180,in=90] ($(n4.west) + (-0.1,0)$);
\end{tikzpicture}
}
\caption{Merging}
\end{subfigure}   
&
  \captionsetup[subfigure]{labelformat=simple,labelsep=period,
                         labelfont=normalfont,textfont=normalfont}
\renewcommand\thesubfigure{\arabic{subfigure}}
 \begin{subfigure}{0.5\linewidth}
 \resizebox{\linewidth}{!}{
\begin{tikzpicture}[scale=1]

\def\valA{0.35}
\def\valB{0.1}
\def\valC{0.675}

    \node [thick, draw=black, circle, draw=black,  heat fill=\valA, , minimum size=5mm] (n1) at (0.5,1.5) {};
    \node [thick, draw=black, circle, draw=black,  heat fill=\valB,, minimum size=5mm] (n2) at (1,0) {};
    \node [thick, draw=black, circle, draw=black,  heat fill=\valC, , minimum size=5mm] (n3) at (3,1) {};
    
    \node  [draw=white, circle] (n0) at (0,0) {};
    \node [draw=white, circle] (n0) at (4,2) {};

    \draw [thick, bend right=+20, ->] (n1) to node[midway, left=0mm ]{\BarGlyph{+1}} (n2);
    \draw [thick, bend left=+20, ->] (n1) to node[midway, above=0mm] {\BarGlyph{+2}} (n3);
    \draw [thick, bend right=+20, ->] (n2) to node[midway, below=0mm] {\BarGlyph{-2}} (n3);

    \draw [thick, bend right=-20, <-] (n1) to (n2);
    \draw [thick, bend left=-20, <-] (n1) to (n3);
    \draw [thick, bend right=-20, <-] (n2) to (n3);
\end{tikzpicture}
}
\caption{Recalibrating}
\end{subfigure} 
\\

\multicolumn{2}{c}{
\begin{tabular}{r@{\,}l r@{\;}l}
\multirow{1}{*}[2.5mm]{\def\colormapheight{1.5cm}\rotatebox{90}{\input{figures/heatcolorbar}}} & node value
&
  \begin{tikzpicture}[baseline,scale=0.2]
  \draw[thick] (0,0) -- (9,0);

  \foreach \i/\h in {
     0/-4,
     1/-3,
     2/-2,
     3/-1,
     4/0,
     5/1,
     6/2,
     7/3,
     8/4
  }{
    \pgfmathsetmacro{\v}{+\h*3/4/9+0.5}%
    \draw[vanimo fill=\v, draw=black, thin] (\i+0.2,0) rectangle (\i+0.8,\h/2);
  }
\end{tikzpicture}

& edge merge gain $\Delta$
\end{tabular}
}

\end{tabular}
}

%% file: figures/heatcolorbar.tex
\begin{tikzpicture}
    \begin{axis}[
        hide axis,
        scale only axis,
        height=5pt,
        width=50pt,
        xlabel={in m},
        colormap name=heat,
        colorbar,
           colorbar,
            colorbar style={
            width=.2cm,
            height=\colormapheight,
            ytick={0,1,2,3,4},
            yticklabels={,,,,},
            yticklabel style={font=\tiny},
            major tick length=1.5pt, 
            line width=.05mm,
            grid style={draw=none} 
        },
        point meta min=0,
        point meta max=4
    ]
    \addplot [draw=none] coordinates {(0,0)};
    \end{axis}
 
\end{tikzpicture}

%% file: figures/quali_over.tex
\small{
\resizebox{\textwidth}{!}{
\begin{tabular}{l@{}c@{}c@{}c@{}c}

\raiseandrotate{0.9}{\small S3DIS} 
&
\includegraphics[width=0.22\linewidth, height=0.15\linewidth] {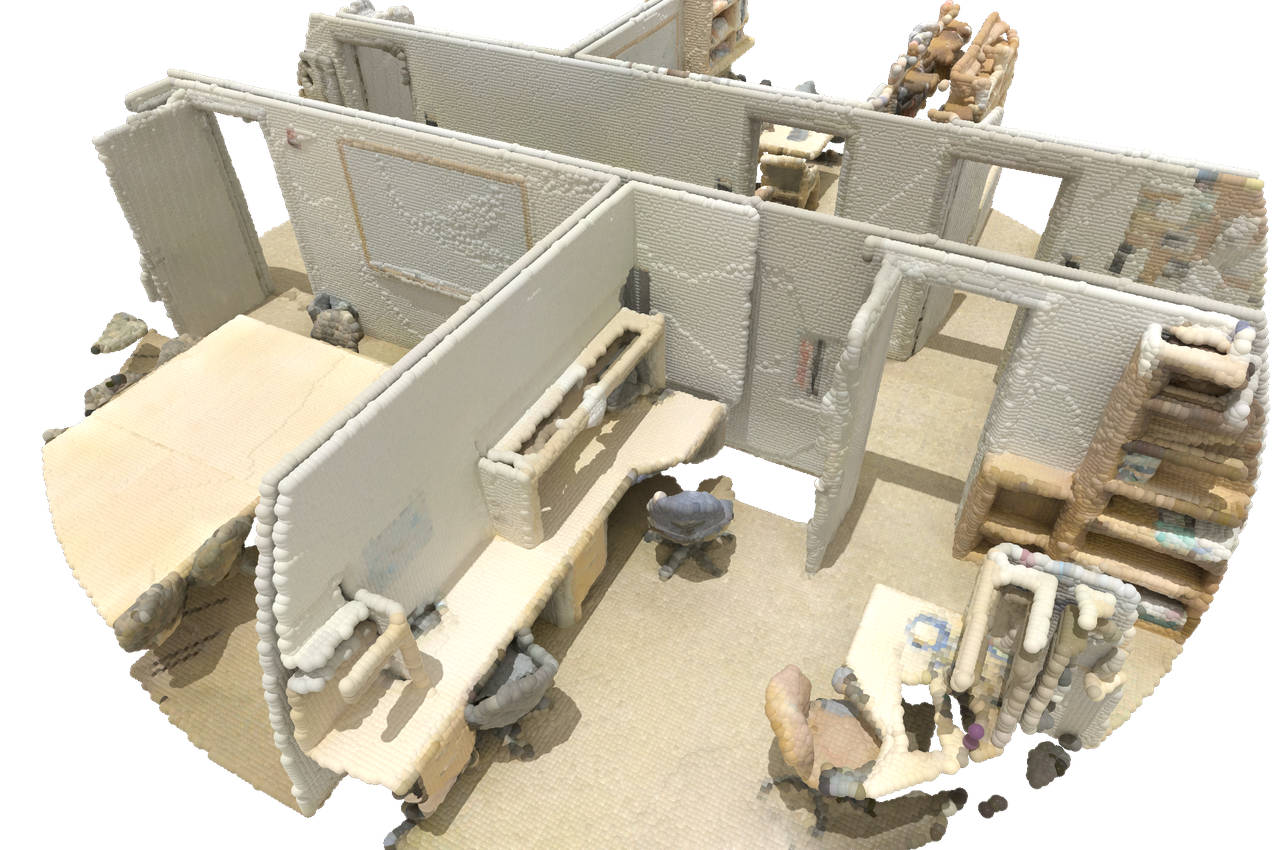}
&
\includegraphics[width=0.22\linewidth, height=0.15\linewidth]{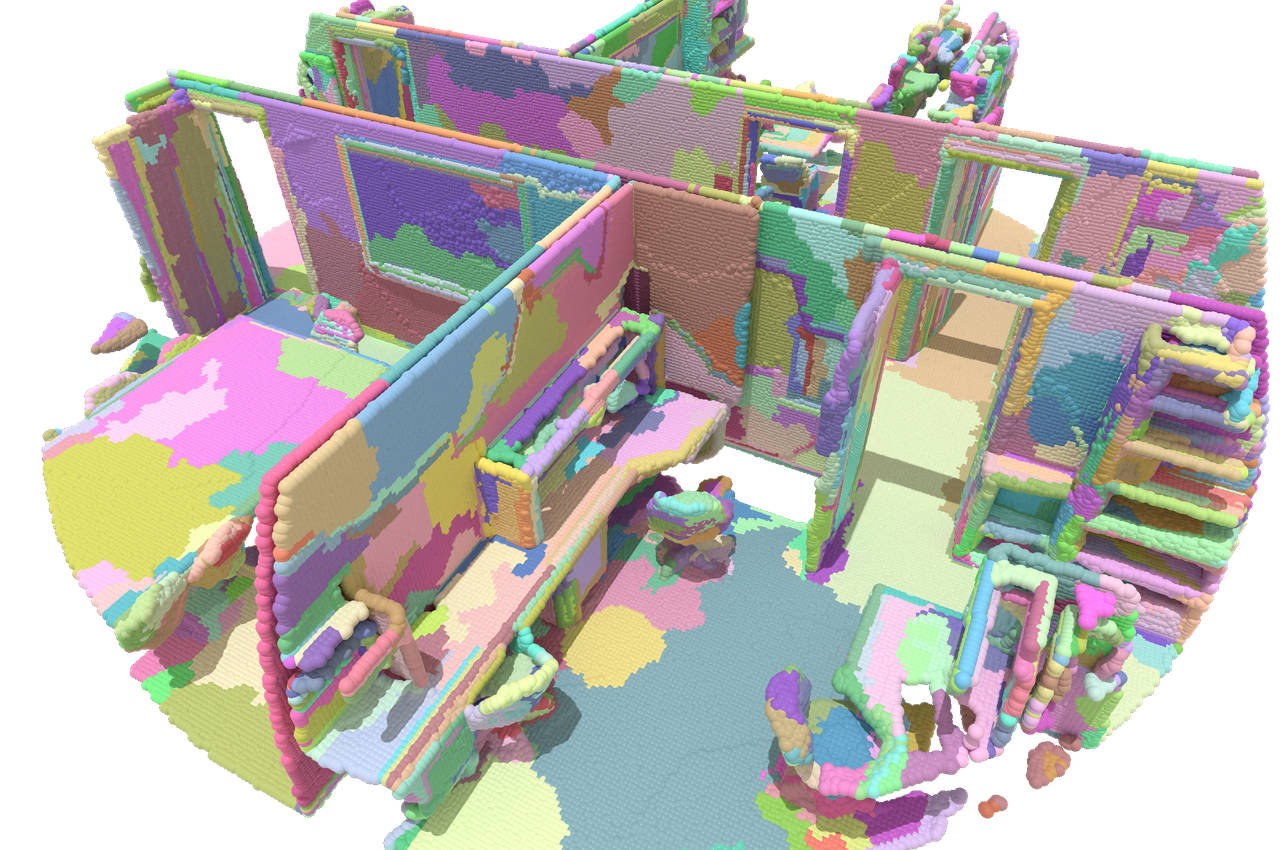}
&
\includegraphics[width=0.22\linewidth, height=0.15\linewidth]{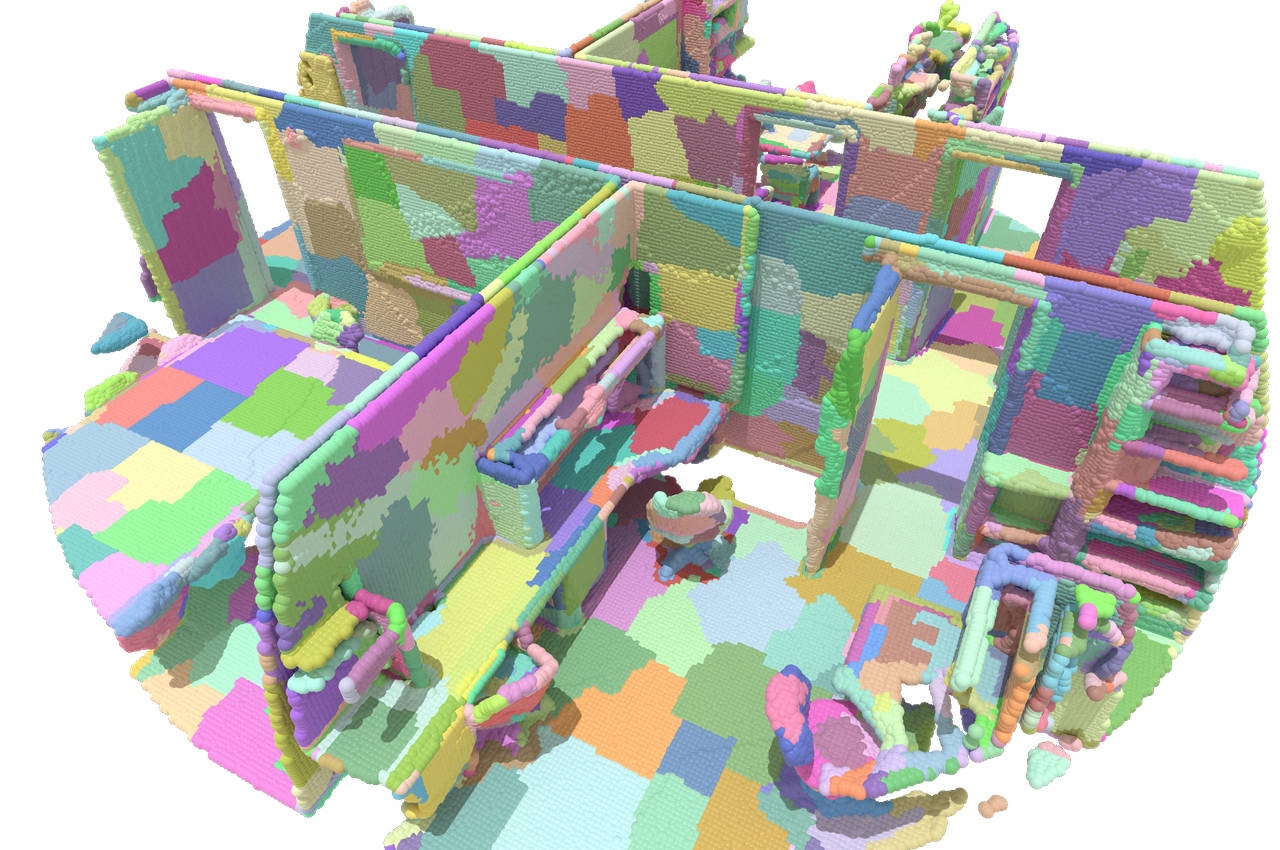}
&
\includegraphics[width=0.22\linewidth, height=0.15\linewidth]{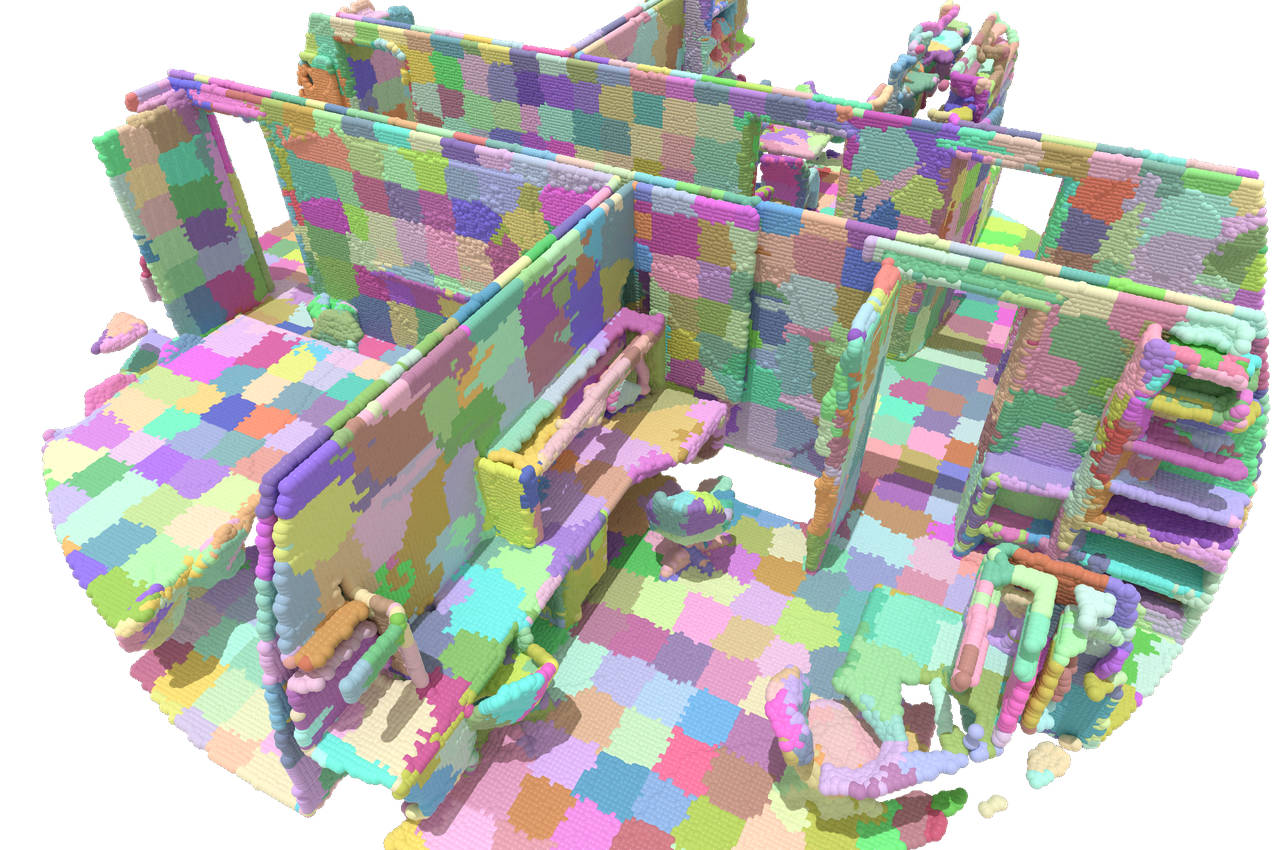}
\\
&
& 267k SP | OmIoU: 93.9
& 265k SP | OmIoU: 91.3
& 282k SP | OmIoU: 90.0
\\[2mm]
\raiseandrotate{0.6}{\small KITTI360} 
&
\includegraphics[width=0.22\linewidth, height=0.15\linewidth]{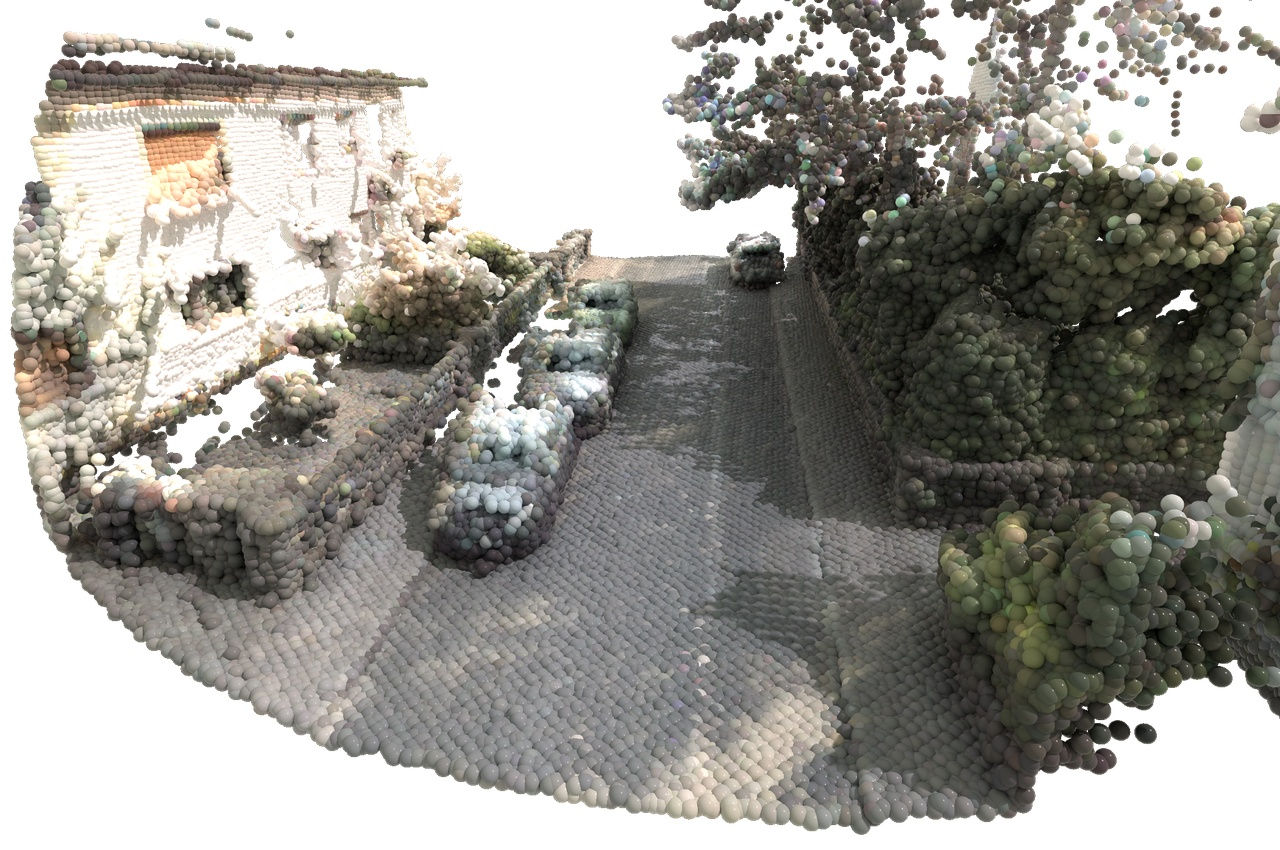}
&
\includegraphics[width=0.22\linewidth, height=0.15\linewidth]{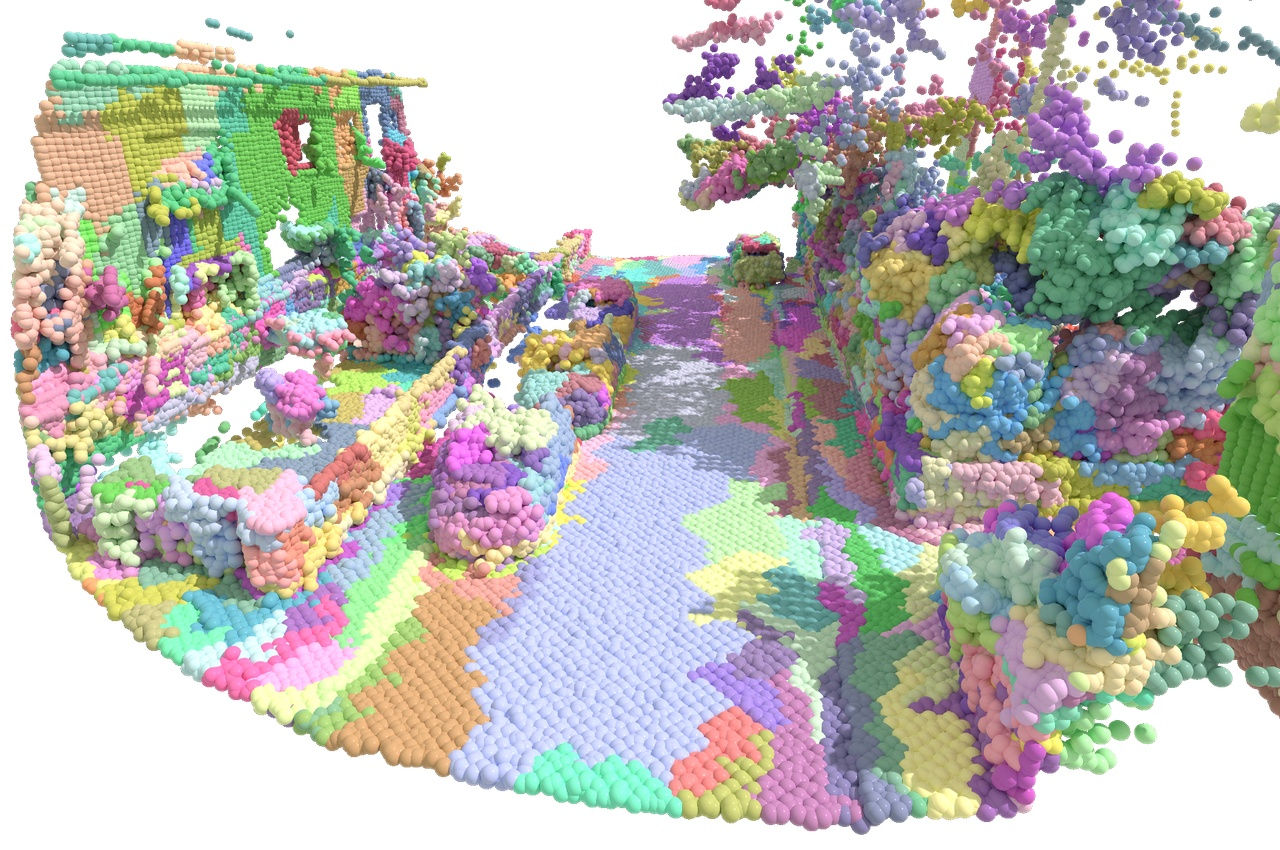}
&
\includegraphics[width=0.22\linewidth, height=0.15\linewidth]{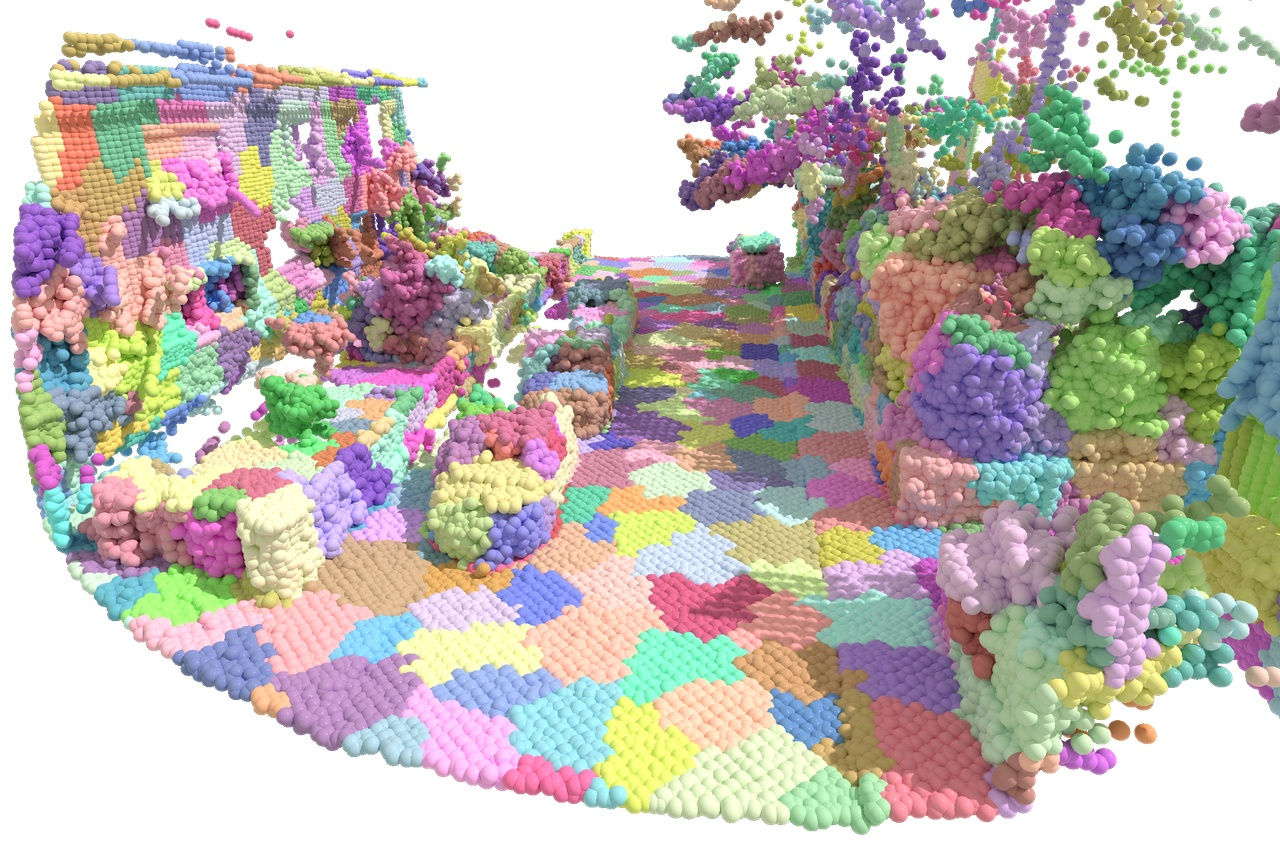}
&
\includegraphics[width=0.22\linewidth, height=0.15\linewidth]{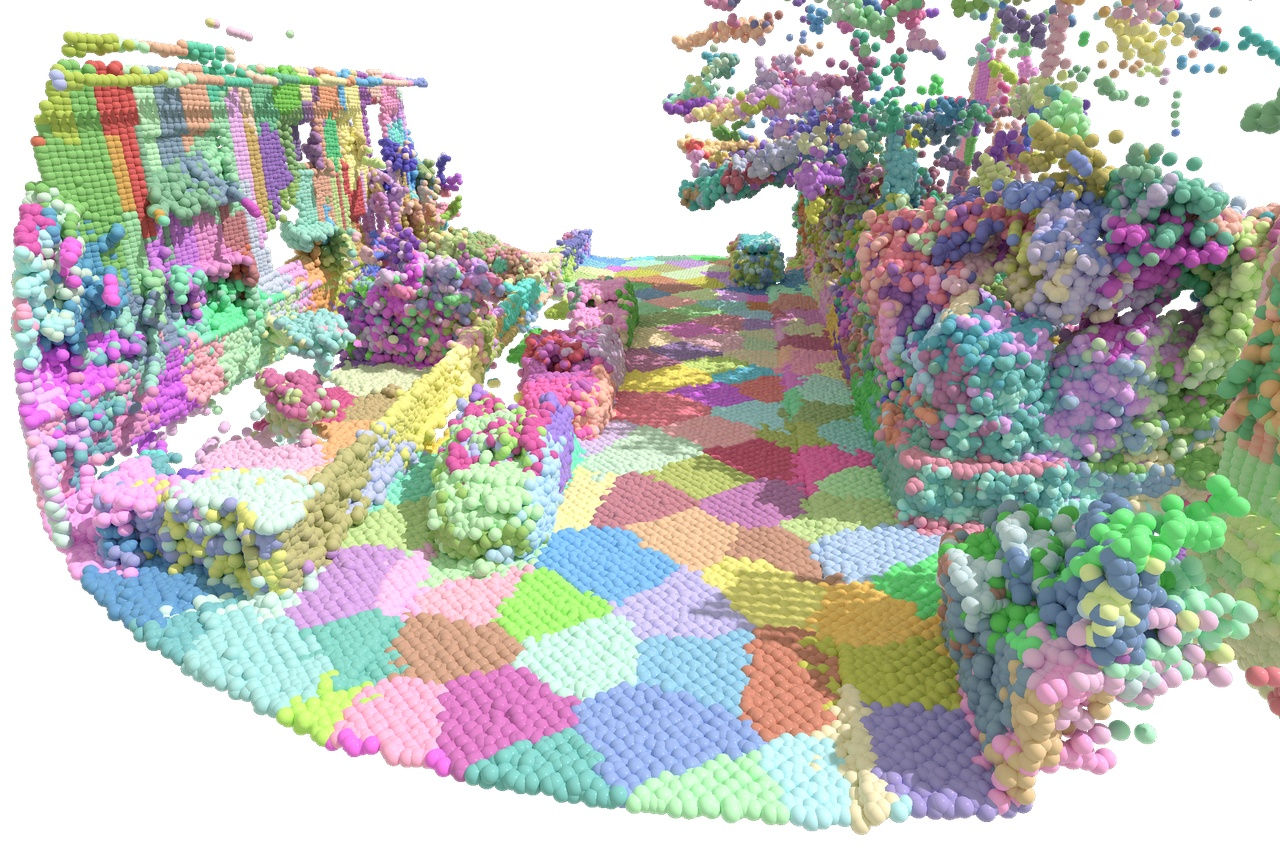}
\\
& 
& 500k SP | OmIoU: 82.8
& 506k SP | OmIoU: 83.2
& 518k SP | OmIoU: 73.6
\\[2mm]
\raiseandrotate{0.7}{\small DALES} 
&
\includegraphics[width=0.22\linewidth, height=0.15\linewidth]{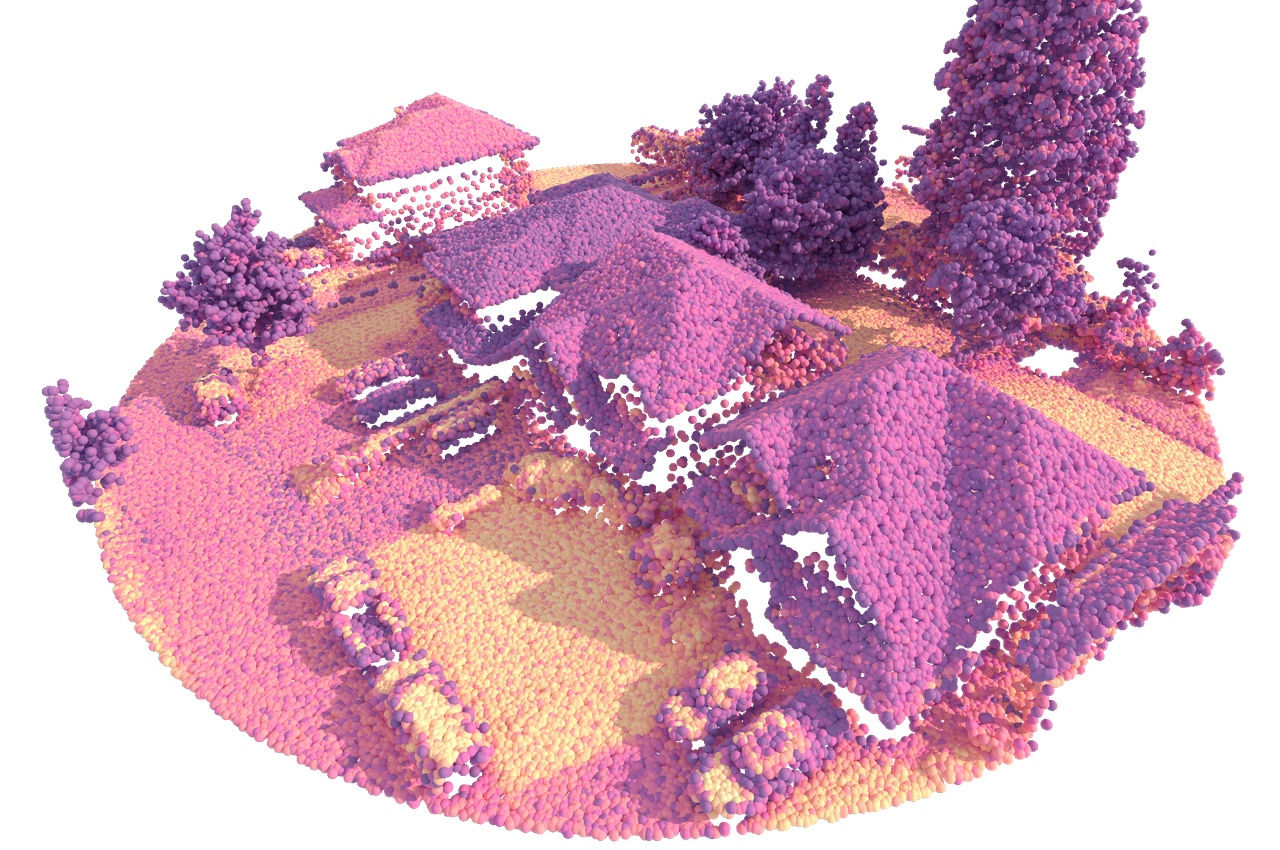}
&
\includegraphics[width=0.22\linewidth, height=0.15\linewidth]{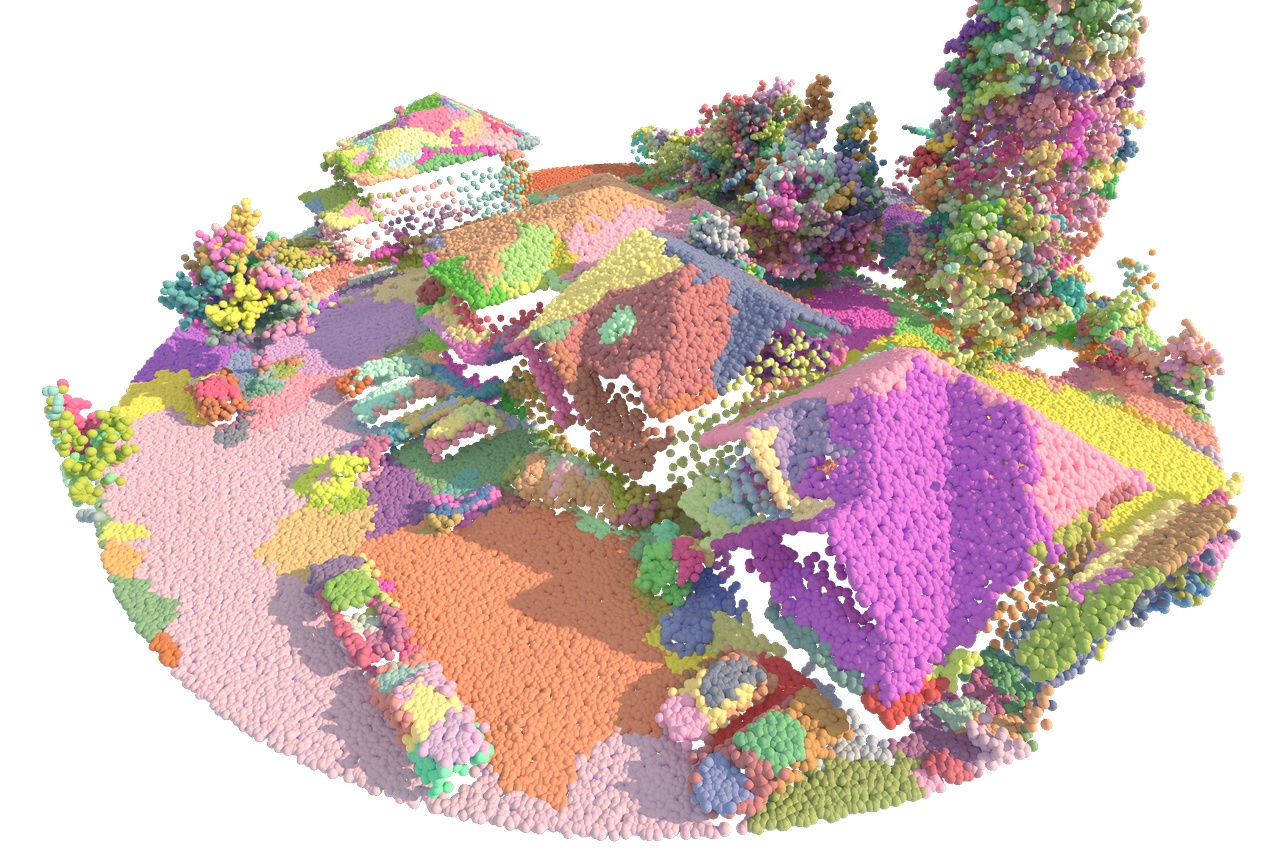}
&
\includegraphics[width=0.22\linewidth, height=0.15\linewidth]{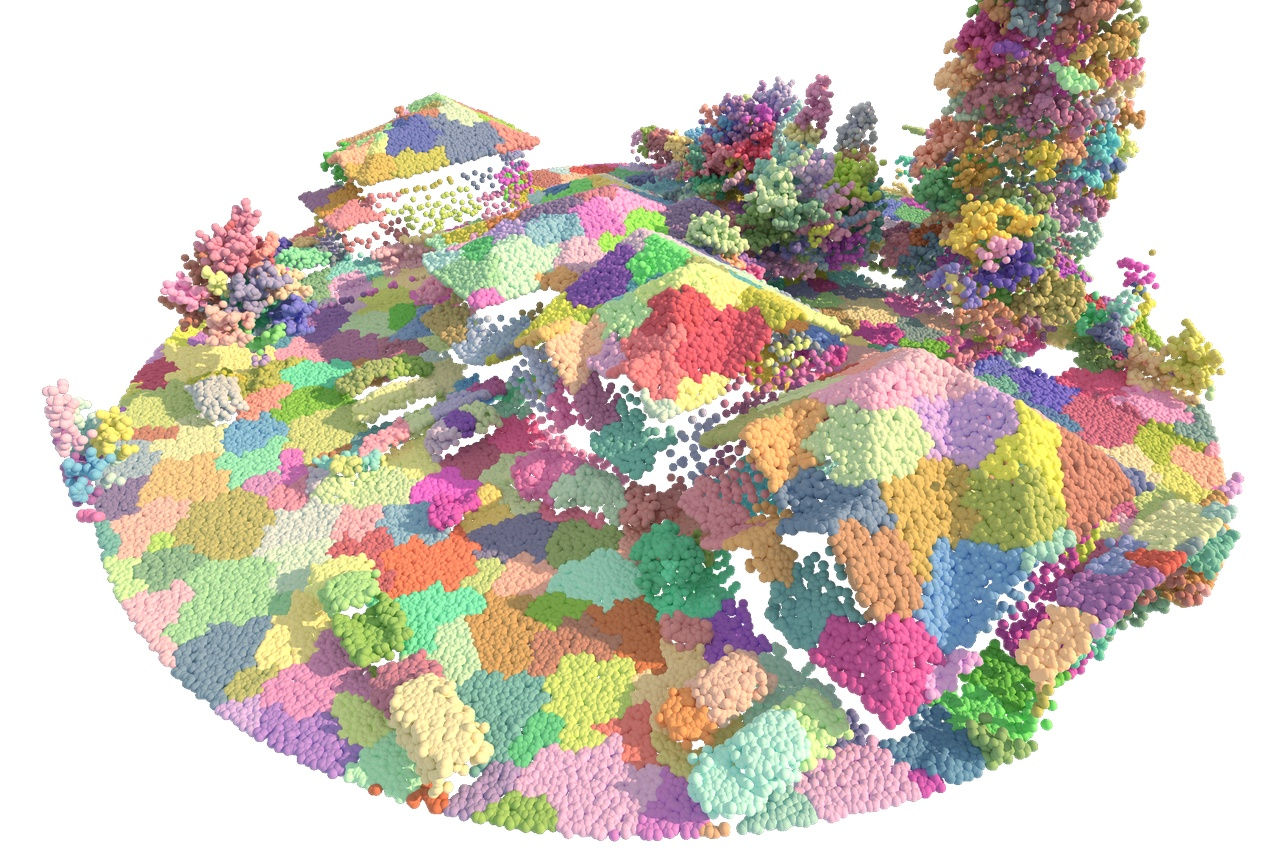}
&
\raisebox{1cm}{\makecell{
VCCS expects RGB\\
DALES has only intensity
}}
\\
& 
& 669k SP | OmIoU: 89.6
& 667k SP | OmIoU: 89.0
& 
\\
&\begin{subfigure}{.25\linewidth}
    \label{fig:quali:input}
    \caption{Input}
\end{subfigure}
&
\begin{subfigure}{.25\linewidth}
    \label{fig:quali:ezsp}
    \caption{\MODELNAME}
\end{subfigure}
&
\begin{subfigure}{.25\linewidth}
    \label{fig:quali:pcp}
    \caption{PCP}
\end{subfigure}
&
\begin{subfigure}{.25\linewidth}
    \label{fig:quali:vccs}
    \caption{VCCS}
\end{subfigure}
\end{tabular}}}

%% file: figures/over.tex
\definecolor{PAPONCOLOR}{HTML}{fb8500}
\definecolor{VOXELCOLOR}{HTML}{59A14F}

\pgfmathsetmacro{\N}{0}
\pgfmathsetmacro{\Mean}{0}
\pgfmathsetmacro{\Mtwo}{0}

\pgfkeys{/pgf/fpu=false}

\begin{tabular}{l@{}c@{}c@{}c@{}c}
\multicolumn{5}{c}{
\begin{tabular}{r@{\,}l r@{\,}l r@{\,}l r@{\,}l r@{\,}l}
\tikz[baseline=-0.5ex] {\draw[EZSPCOLOR, very thick] (0,0) -- (0.6,0); \node [circle, fill=EZSPCOLOR, minimum width=1.5mm, scale=0.75] at (0.3,0){};}
&
\small \bf \MODELNAME (ours)
&
\tikz[baseline=-0.5ex] {\draw[SPNETCOLOR, very thick] (0,0) -- (0.6,0); \node [regular polygon, regular polygon sides=4, fill=SPNETCOLOR, minimum width=1.25mm, scale=0.75] at (0.3,0){};}
&
\small SPNet \cite{hui2021superpoint}
&
\tikz[baseline=-0.5ex] {\draw[SPTCOLOR, very thick] (0,0) -- (0.6,0);  \node [regular polygon, regular polygon sides=3, fill=SPTCOLOR, minimum width=0.5mm, scale=0.5] at (0.3,-0.02){};}
&
\small PCP \cite{raguet2019parallel,robert2023efficient}
&
\tikz[baseline=-0.5ex] {\draw[PAPONCOLOR, very thick] (0,0) -- (0.6,0); \node [diamond, fill=PAPONCOLOR, minimum width=1.0mm, scale=0.75] at (0.3,0){};}
&
\small VCCS \cite{papon2013voxel}
&
\tikz[baseline=-0.5ex] {\draw[VOXELCOLOR, very thick] (0,0) -- (0.6,0); \node [regular polygon, fill=VOXELCOLOR, minimum width=1.5mm, scale=0.75] at (0.3,0){};}
&
\small Voxelisation
\end{tabular}
}\\
\raiseandrotate{1.4}{\small Oracle mIoU}
&
\begin{subfigure}{.24\linewidth}
 \begin{tikzpicture}
  \begin{axis}[
    width=0.8\textwidth,
    height=3cm,
    scale only axis,
    xmode=log,
    log basis x=10,
    xmin=2e4, xmax=7e5,
    ymin=80, ymax=100,
    xlabel={\small \# of superpoints},
    tick align=inside,
    grid=both,
    minor tick num=1,
    major grid style={opacity=0.3},
    minor grid style={opacity=0.1},
  ]
    \addplot+[mark=*, thick, smooth, EZSPCOLOR, mark options={solid, fill=EZSPCOLOR}] table [x=sp, y=omiou] {figures/data/over/ezsp_s3dis_val_hec.csv};
    
    \addplot+[mark=square*, thick, smooth, SPNETCOLOR, mark options={solid, fill=SPNETCOLOR}] table  [x=sp, y=omiou] {figures/data/over/spnet_s3dis_val_hec.csv};
    
    \addplot+[mark=triangle*, thick, smooth, SPTCOLOR, mark options={solid, fill=SPTCOLOR}] table [x=sp, y=omiou] {figures/data/over/pcp_s3dis_val_hec.csv};
    
    \addplot+[mark=diamond*, thick, smooth, PAPONCOLOR, mark options={solid, fill=PAPONCOLOR}] table [x=sp, y=omiou] {figures/data/over/vccs_s3dis.csv};
    
    \addplot+[mark=pentagon*, smooth, VOXELCOLOR, mark options={solid,fill=VOXELCOLOR}] table [x=sp, y=omiou] {figures/data/over/VoxelPartition_s3dis_val.csv};
  \end{axis}
\end{tikzpicture}
    \label{fig:over:s3dis}
    \vspace{-5mm}
    \caption{S3DIS}
\end{subfigure}
&
\begin{subfigure}{.24\linewidth}
 \begin{tikzpicture}
  \begin{axis}[
     width=0.8\textwidth,
    height=3cm,
    scale only axis,
    xmode=log,
    log basis x=10,
    xmin=1e5, xmax=1e7,
    ymin=70, ymax=100,
    xlabel={\small \# of superpoints},
    tick align=inside,
    grid=both,
    minor tick num=1,
    major grid style={opacity=0.5},
    minor grid style={opacity=0.3},
  ]
    
    \addplot+[mark=*, thick, smooth, EZSPCOLOR, mark options={solid, fill=EZSPCOLOR}] table [x=sp, y=omiou] {figures/data/over/ezsp_kitti360_val_hec.csv};
    
    \addplot+[mark=triangle*, thick, smooth, SPTCOLOR, mark options={solid, fill=SPTCOLOR}] table [x=sp, y=omiou] {figures/data/over/pcp_v2_kitti360_val_hec.csv};

    \addplot+[mark=diamond*, thick, smooth, PAPONCOLOR, mark options={solid, fill=PAPONCOLOR}] table [x=sp, y=omiou, col sep=comma] {figures/data/over/vccs_kitti360.csv};
    
    \addplot+[mark=pentagon*, smooth, VOXELCOLOR, mark options={solid,fill=VOXELCOLOR}] table [x=sp, y=omiou] {figures/data/over/VoxelPartition_kitti360_val.csv};
  \end{axis}
  
\end{tikzpicture}
    \label{fig:over:kitti}
    \vspace{-5mm}
    \caption{KITTI360} 
\end{subfigure} 
&
\begin{subfigure}{.24\linewidth}
 \begin{tikzpicture}
  \begin{axis}[
    width=0.8\textwidth,
    height=3cm,
    scale only axis,
    xmode=log,
    log basis x=10,
    xmin=1e5, xmax=1e7,
    ymin=70, ymax=100,
    xlabel={\small \# of superpoints},
    tick align=inside,
    grid=both,
    minor tick num=1,
    major grid style={opacity=0.3},
    minor grid style={opacity=0.1},
  ]
    \addplot+[mark=*, thick, smooth, EZSPCOLOR, mark options={solid, fill=EZSPCOLOR}] table [x=sp, y=omiou] {figures/data/over/ezsp_dales_val_hec.csv};
    
    \addplot+[mark=triangle*, thick, smooth, SPTCOLOR, mark options={solid, fill=SPTCOLOR}] table [x=sp, y=omiou] {figures/data/over/pcp_v2_dales_val_hec.csv};

    \addplot+[mark=pentagon*, smooth, VOXELCOLOR, mark options={solid,fill=VOXELCOLOR}] table [x=sp, y=omiou] {figures/data/over/VoxelPartition_dales_val.csv};
    
  \end{axis}
\end{tikzpicture}
    \label{fig:over:dales}
    \vspace{-5mm}
    \caption{DALES} 
\end{subfigure}
&
\begin{subfigure}{.24\linewidth}
\begin{tikzpicture}
  \begin{axis}[
    width=1\textwidth,
    height=4.55cm,
    ymode=log,
    ylabel={\small Throughput pt/s},
    ylabel style={yshift=-3mm},
    xtick={1,2,3,4,5},
    xticklabels={\small Voxels, \small \MODELNAME, \small VCSS, \small PCP, \small SPNet},
    xlabel={\vphantom{\small \# of superpoints}},
    ymajorgrids,
    bar width=12pt,
    enlarge x limits=0.12,
    ymin=1e5, ymax=5e7, 
    x tick label style={rotate=45, anchor=east, align=center},
  ]
\plotbarerror{1}{figures/data/over/VoxelPartitionPreVoxelized_s3dis_val_hec.csv}{VOXELCOLOR}

\plotbarerror{2}{figures/data/over/ezsp_s3dis_val_hec_2.csv}{EZSPCOLOR}

\plotbarerror{3}{figures/data/over/vccs_s3dis.csv}{PAPONCOLOR}

\plotbarerror{4}{figures/data/over/pcp_s3dis_val_hec_2.csv}{SPTCOLOR}

\plotbarerror{5}{figures/data/over/spnet_s3dis_val_hec.csv}{SPNETCOLOR}

  \end{axis}
\end{tikzpicture}
 \label{fig:over:throughput}
    \vspace{-0.9mm}
    \caption{Computational Effiency}
       
\end{subfigure}
\end{tabular}

%% file: sections/4_experiments.tex
We evaluate \MODELNAME~on three large-scale 3D segmentation benchmarks.  
We first outline the experimental setup (\cref{sec:setting}), then assess partition quality and efficiency (\cref{sec:overseg}).  
Next, we report semantic segmentation performance with a downstream superpoint classifier (\cref{sec:semseg}), followed by an ablation study (\cref{sec:ablation}).

\subsection{Experimental Setting}
\label{sec:setting}
\label{sec:data}
\para{Datasets.}
We evaluate on three datasets covering various sensing modalities and scales:  
\begin{compactitem}
\item \textbf{S3DIS}~\cite{armeni20163d}: Indoor scans of six large building floors, totaling $273$M points annotated in $13$ classes.  
Following~\cite{thomas2019kpconv}, we use the \emph{merged} version, where each floor is treated as a single point cloud.  

\item \textbf{KITTI-360}~\cite{liao2022kitti}: Mobile mapping LiDAR with $919$M points and 15 classes, spanning $300$ large-scale urban scenes, including $61$ validation scans.  

\item \textbf{DALES}~\cite{varney2020dales}: Aerial LiDAR over urban and suburban areas with $492$M points across $8$ classes, comprising $40$ scans, $12$ reserved for evaluation.  
\end{compactitem}
We subsample all point clouds on a regular grid: 3\,cm for S3DIS, 10\,cm for KITTI-360 and DALES.

\para{Implementation Details.}  
For partitioning, we fix $\tau{=}1$ in \cref{eq:softmax}, $w_{p,q}{=}1$, and build adjacency from the 8-nearest neighbors.  
The regularization in \cref{eq:gmpp} is $\lambda{=}0.02$, and the intra-edge sampling ratio in \cref{eq:loss} is $\rho_\text{intra}{=}0.1$ for S3DIS and $0.3$ for KITTI-360/DALES.   

The backbone $\phi^\text{point}$ is a sparse CNN~\cite{graham20183d} implemented with TorchSparse~\cite{tangandyang2023torchsparse}, with three layers of width $[32,32,32]$.  
Kernels are $3^3$ except in the first layer for KITTI-360/DALES ($7^3$).  
The embedding dimension is $M{=}32$, yielding models of 58k (S3DIS), 89k (KITTI-360), and 67k (DALES) parameters.
We compute three-level hierarchical partitions with minimal superpoint sizes of $[5,30,90]$ (S3DIS/KITTI-360) and $[5,15,70]$ (DALES).

For segmentation, we use a modified SPT-64 on S3DIS and DALES, and an SPT-128 on KITTI-360: we add a third hierarchical stage, reinstate the feed-forward layer in the transformer block (for DALES only), and concatenate  color, position, and elevation with the CNN feature map to form the point embeddings. 
This yields segmentation heads with $330$k (S3DIS), $870$k (KITTI-360), and $425$k (DALES) parameters.
To fight class imbalance, we trained with a focal loss \cite{lin2017focal} ($\gamma=1$ for S3DIS and $\gamma=2$ for KITTI-360/DALES).
%
All models are trained with Adam (default hyperparameters) and cosine learning rate scheduling with 50 warm-up epochs.

\subsection{Oversegmentation Results}
\label{sec:overseg}
We compare the quality of \MODELNAME's partitions against classical and learning-based oversegmentation methods.

\para{Metrics.}  
We follow the common practice of evaluating \emph{oversegmentation}, \ie partitioning a point cloud into compact regions that ideally align with semantic objects.  
Since the downstream classifier operates on superpoints, partition quality is measured by two criteria: the \emph{number of superpoints} produced and the \emph{oracle mIoU}~\cite{landrieu2019point}, defined as the mIoU obtained by assigning each superpoint its majority ground-truth label.  
This provides an upper bound on the segmentation accuracy achievable with a given partition.
We also report throughput, obtained either from official model logs or from our own re-runs. All measurements are conducted on comparable Ampere- or Ada-generation GPUs.  

\para{Baselines and Competing Methods.}
We compare against:
\begin{compactitem}
\item \textbf{Voxelization:} Uniform voxel grid grouping.
\item \textbf{VCCS}~\cite{papon2013voxel}: A classical voxel-based oversegmentation based on $k$-means.
\item \textbf{Parallel Cut Pursuit (PCP)}~\cite{raguet2019parallel}: The updated graph-cut partitioner~\cite{landrieu2017cut} used in SPT~\cite{robert2023efficient}.
\item \textbf{SPNet}~\cite{hui2021superpoint}: Learns partitions via differentiable $k$-means.
\end{compactitem}

\begin{table*}[t]
\caption{
\textbf{Efficiency and Performance.}  
End-to-end processing time on the full S3DIS dataset ($273$M points), broken down by stage.  
We also compare model size and semantic segmentation performance on S3DIS, KITTI-360, and DALES.  
}
\label{tab:bigtable}
\centering
\vspace{-2mm}
\input{tables/bigtable}
\end{table*}

\para{Results.}
\Cref{fig:over} reports the purity of the partition obtained by all methods on three benchmarks, along with their throughput.  
\MODELNAME~consistently achieves a purity higher or on par with the best oversegmentation methods for the same number of superpoints, while being a full \emph{order of magnitude faster}.
Although our greedy solver yields an objective value roughly 25\% higher than PCP when minimizing~\cref{eq:gmpp}, the resulting partitions exhibit comparable semantic purity in practice.
The purity of VCCS's and SPNet's superpoints is limited by their reliance on the rigid $k$-means algorithm. Moreover, VCCS's  CPU-based implementation limits its throughput.  
SPNet, despite being trained for partitioning, underperforms in purity and requires $\sim$6\,h of training, against fewer than ${20}$ minutes for \MODELNAME. Voxelization naturally remains the fastest partitioning algorithm, but also yields the least semantically pure partitions. 

\para{Qualitative Analysis.} We report qualitative examples of partitions in \cref{fig:quali_over}.  
The $k$-means–based method VCCS fails to adapt to local complexity and produces partitions that resemble uniform tessellations.  
PCP adapts better to variations in complexity, but still tessellates large, simple surfaces.  
In contrast, \MODELNAME~yields the most adaptive partitions: it produces large, semantically pure superpoints on simple structures such as ground, roofs, or walls, while allocating small superpoints to geometrically complex regions.  

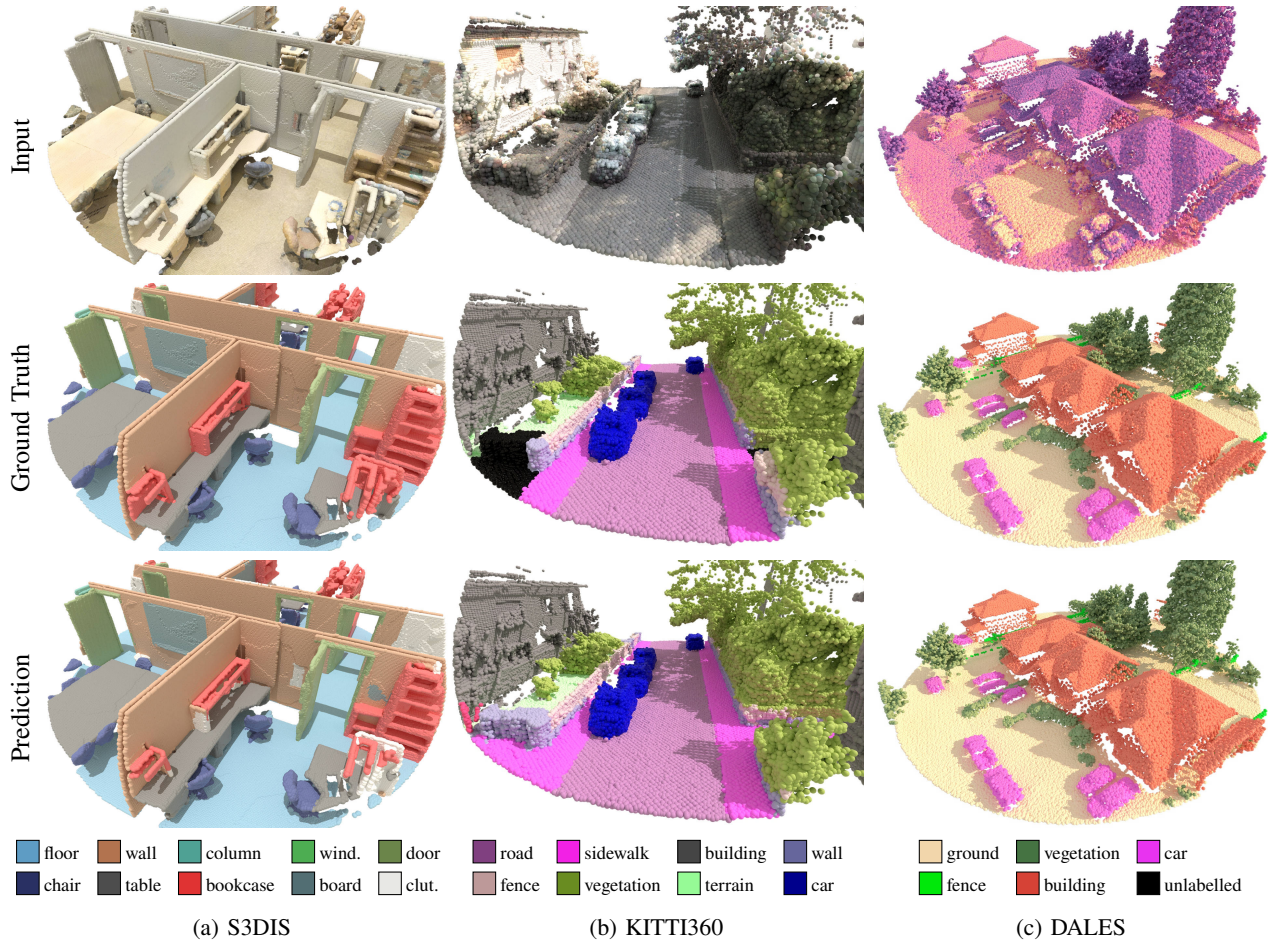
\begin{figure*}
\centering
\input{figures/quali_seg}

\vspace{-2mm}
\caption{
\textbf{Semantic Segmentation.} 
Qualitative results on three benchmarks: 
S3DIS, KITTI-360, and DALES. 
For each dataset, we show the input point cloud (RGB for S3DIS and KITTI-360, LiDAR intensity for DALES), the ground-truth labels, and the predictions of \MODELNAME.
}
\label{fig:quali_seg}
\end{figure*}

\subsection{Semantic Segmentation}
\label{sec:semseg}
\Cref{tab:bigtable} reports both accuracy and efficiency for a range of SOTA semantic segmentation models. 
All methods are trained only on the official training split of each dataset, without using any external data.

\para{Analysis.}  
Superpoint-based methods consistently deliver competitive accuracy with \emph{100--200$\times$} fewer parameters than typical point- or voxel-based networks.  
Yet their inference speed is often constrained by the CPU-bound partitioning stage, which can dominate runtime.  
\MODELNAME~removes this bottleneck, yielding much faster end-to-end inference while retaining top-tier accuracy.  
With fewer than $400$k parameters (under $2$\,MB of VRAM), it ranks among the smallest models in this comparison---orders of magnitude lighter than many voxel- or point-based alternatives.  
In terms of inference speed, \MODELNAME~is unmatched: faster even than PointNet++, while maintaining strong performance across all datasets.  
Its accuracy trails SOTA large models by less than $2$ mIoU points, comparable to SPT and well within the expected variance of training and macro-averaged evaluation metrics.

In contrast, modern point- and voxel-based architectures such as PointNeXt~\cite{qian2022pointnext}, Stratified Transformer~\cite{lai2022stratified}, and PointTransformer-v3~\cite{wu2024point} can be extremely slow at inference, in part due to heavy test-time augmentation (TTA).  
Removing TTA reduces accuracy by $-1.5$ mIoU on S3DIS Area~5 for Stratified Transformer and by $-1.6$ for PTv3, yet they still remain significantly slower than our approach (e.g., PTv3 requires $988$\,s on S3DIS without TTA).  

\para{Qualitative Analysis.}  
\Cref{fig:quali_seg} shows semantic segmentations produced by \MODELNAME~across indoor, outdoor, and aerial domains.  
Predictions remain reliable even in complex environments, with most errors arising from the ambiguous \texttt{clutter} class, and cases that are inherently difficult to partition into superpoints due to low geometric and radiometric contrasts at their interface, such as a whiteboard against a white wall.  
At the same time, \MODELNAME~accurately captures fine details such as furniture edges, vehicle boundaries, and roof structures.  
Overall, it achieves high-quality segmentation while delivering orders of magnitude faster inference than existing approaches.

\para{Breaking the Partition Bottleneck.}
\Cref{fig:bottleneck} details the breakdown of end-to-end inference time across different pipeline stages for Superpoint Transformer \cite{robert2023efficient} and our proposed \MODELNAME---excluding I/O, which largely depends on dataset format, \eg plain \texttt{.txt} vs.\ binary \texttt{.ply} files.
For SPT, the CPU-bound partition stage dominates computation time; in \MODELNAME, this cost is virtually eliminated due to our GPU-based implementation and the removal of handcrafted point features.  
Additional minor optimizations further reduce runtime in other stages.  
Beyond speed, \MODELNAME~is also easier to deploy and tune, as it avoids the complex feature engineering and hyperparameters optimization required by traditional superpoint partitioning methods.

\para{Memory Efficiency.}   
As shown in \cref{fig:spectrum}, \MODELNAME~processes full scans \emph{in a single pass, without tiling}, across scenarios ranging from real-time inference on a single Velodyne64 sweep to city-scale aerial surveys. 
A single LiDAR rotation runs on an embedded Jetson device ($<\$200$), while an entire S3DIS floor (68 rooms) fits on a consumer GPU ($<\$1000$).  
Even 1.3\,km$^2$ of aerial LiDAR from DALES can be handled at once on an NVIDIA A40 ($\sim\$3000$).  

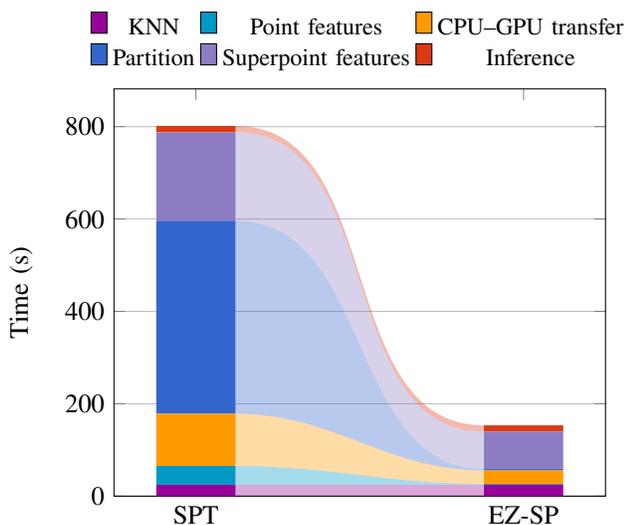
\begin{figure}
\centering
\input{figures/bottleneck}
\vspace{-2mm}
\caption{
\textbf{Computation Breakdown.}  
Breakdown of runtime comparison between SPT and \MODELNAME~for S3DIS 6-Fold. 
}
\label{fig:bottleneck}
\vspace{-0.5em}
\end{figure}


\begin{table}
\caption{
\textbf{Scalability of \MODELNAME.}  
From a single LiDAR scan to a city-scale aerial survey, entire scenes can be processed in one pass on embedded or commercial-grade GPU memory.
}
    \label{fig:spectrum}
    \centering
    \input{tables/spectrum}
\end{table}

\subsection{Ablation Study}
\label{sec:ablation}
We assess the contribution of each component of our approach through the ablations summarized in \cref{tab:ablation}. 
For ablations A, B, and D, we retrain our SPT model from scratch.  

\begin{compactitem}
    \item \textsc{A}: \textbf{Learning to Partition.}  
    Replacing our lightweight CNN (\cref{sec:overseg}) with handcrafted geometric features~\cite{demantke2011dimensionality,guinard2017weakly} yields comparable performance. This demonstrates that our approach eliminates the need for handcrafted inputs, as the CNN learns features of similar expressivity, while reducing manual engineering.
    
    
    \item \textbf{B: Hierarchical Levels.} 
    Adding a third hierarchical level to our SPT improves accuracy with negligible impact on throughput. 
    
    \item \textsc{C}: \textbf{Optimizations.} The various implementation optimizations over the original SPT codebase (CPU-GPU transfer, GPU-bound feature computation) do not impact the performance but double our throughput.    
    
    \item \textbf{D: Partition Algorithm.} 
    Replacing our GPU partition algorithm with the original CPU-based PCP~\cite{raguet2019parallel} based on handcrafted features (as in SPT) reduces throughput substantially, while delivering slightly lower accuracy. 
    As in Robert et al.~\cite{robert2023efficient}, we find that the 3-level PCP partition does not improve performance, unlike ablation B, suggesting that
    despite comparable oracle mIoU,
    our superpoint partition’s hierarchical structure is more informative than PCP’s.
    



\end{compactitem}

\begin{table}[]
\caption{{\bf Ablation.} Performance of variants of \MODELNAME.}
\label{tab:ablation}
\centering
\vspace{-2mm}

\input{tables/ablation}
\end{table}

%% file: tables/bigtable.tex
\centering\small{
\begin{tabular}{r@{ : }l r@{ : }l r@{ : }l r@{ : }l}
     \textcolor{cPreproc}{\faCog}&preprocessing& \textcolor{cPartition}{\faCut}&partition&
     \textcolor{cInference}{\faRobot}&semantic segmentation&
     \textcolor{black}{\faStopwatch}&total time 
\end{tabular}

\begin{tabular}{l@{\;\;\;}l ccccc cc c c}
    \toprule
    \multicolumn{2}{l}{\multirow{3}{*}{Model}} &
    \multicolumn{4}{c}{Inference time (in GPU-s) $\downarrow$} &
    Size $\downarrow$ &
    \multicolumn{4}{c}{Performance (mIoU) $\uparrow$} \\
    
    \cmidrule(lr){3-6}
    \cmidrule(lr){7-7}
    \cmidrule(lr){8-11}
    
    && \multirow{2}{*}[-1mm]{\textcolor{cPreproc}{\faCog}}
       & \multirow{2}{*}[-1mm]{\textcolor{cPartition}{\faCut}}
       & \multirow{2}{*}[-1mm]{\textcolor{cInference}{\faRobot}}
       & \multirow{2}{*}[-1mm]{\textcolor{black}{\faStopwatch}}
       & \multirow{2}{*}[-1mm]{\makecell{$\times10^6$\\params.}}
       & \multicolumn{2}{c}{S3DIS}
       & \multicolumn{1}{c}{K-360}
       & \multirow{1}{*}{DALES} \\
       
    \cmidrule(lr){8-9} \cmidrule(lr){10-10} \cmidrule(lr){11-11}
    
    && &&&&
       & \footnotesize{6-Fold}
       & \footnotesize{Area~5}
       & \multicolumn{1}{c}{val}
       & test \\
       
    \midrule 
    
    \multirow{6}{*}{\rotatebox{90}{point/voxel}} 
      & PointNet++ \cite{qi2017pointnetpp} 
        & \hphantom{3}125   
        & -      
        & \hphantom{K}52  
        & \hphantom{K}\negphantom{1}177 
        & \hphantom{4}3.0  
        & 56.7              
        & -      
        & -     
        & 68.3 
        \\
        
      & KPConv \cite{thomas2019kpconv} 
        & 1031 
        & -      
        & \hphantom{K}\negphantom{3}354  
        & \hphantom{K}\negphantom{13}1385
        & 14.1
        & 70.6              
        & 67.1   
        & -     
        & \bf 81.1 
        \\
        
      & MinkowskiNet \cite{choy20194d} 
        & \hphantom{3}887   
        & -      
        & \negphantom{3}\hphantom{K}302 
        & \hphantom{K}\negphantom{11}1189  
        & 37.9
        & 69.1              
        & 65.4   
        & 58.3  
        & - 
        \\
        
      & PointNeXt-XL \cite{qian2022pointnext} 
        & -                 
        & -      
        & 66k              
        & 66k            
        & 41.6
        & 74.9              
        & 71.1   
        & -     
        & - 
        \\
        
      & Strat. Trans. \cite{lai2022stratified}
        & -                 
        & -      
        & 26k   
        & 26k 
        & \hphantom{4}8.0
        & 74.9              
        & 72.0   
        & -     
        & 74.3 
        \\
        
      & PTv3 \cite{wu2024point}  
        & -                 
        & -      
        & 11k   
        & 11k  
        & 46.2
        & \bf 77.7          
        & \bf 73.4 
        & -   
        & - 
        \\
        
    \greyrule
    
    \multirow{5}{*}{\rotatebox{90}{superpoint}} 
      & SPG \cite{landrieu2018spg} 
        & 3187              
        & 2616   
        & \hphantom{K}56  
        & \hphantom{K}\negphantom{58}5859  
        & \hphantom{3}0.28
        & 62.1              
        & 58.0   
        & -     
        & 60.6 
        \\
        
      & SSP \cite{landrieu2019point}  
        & 3220
        & 2616
        & \hphantom{K}56
        & \hphantom{K}\negphantom{58}5892
        & \hphantom{1}0.29
        & 68.4              
        & 61.7   
        & -     
        & -  
        \\
        
      & SPNet \cite{hui2021superpoint} 
        & 3187
        & \hphantom{2}445
        & \hphantom{K}56
        & \hphantom{K}\negphantom{36}3688
        & \hphantom{1}0.33
        & 68.7              
        & -      
        & -     
        & - 
        \\
        
      & SPT \cite{robert2023efficient}  
        & \hphantom{1}376
        & \hphantom{1}418
        & \bf\hphantom{K}14
        & \hphantom{K}\negphantom{8}808
        & \bf\hphantom{1}0.21
        & 76.0              
        & 68.9   
        & \bf 63.5 
        & 79.6 
        \\
        
      & \bf \MODELNAME (ours)
        & \bf\hphantom{1}136  
        & \bf\hphantom{244}3
        & \bf\hphantom{K}14
        & \bf\hphantom{K}\negphantom{1}153
        & \hphantom{1}0.39
        & 76.1
        & 69.6
        & 62.0
        & 79.4
        \\
      
    \bottomrule
\end{tabular}

}

%% file: figures/quali_seg.tex
\newcommand{\sdisInput}{images/prediction/s3dis_pcp_Area_5__center=2_15_0__radius=5_rgb_colors.jpg}
\newcommand{\sdisGT}{images/prediction/s3dis_pcp_Area_5__center=2_15_0__radius=5_y_colors.jpg}
\newcommand{\sdisPred}{images/prediction/s3dis_ezsp_Area5__center=2_15_0__radius=5_pred_colors.jpg}

\newcommand{\kittiInput}{images/prediction/kitti360_ezsp_0000000372_0000000610__center=-150_34_-2__radius=15_rgb_colors.jpg}
\newcommand{\kittiGT}{images/prediction/kitti360_ezsp_0000000372_0000000610__center=-150_34_-2__radius=15_y_colors.jpg}
\newcommand{\kittiPred}{images/prediction/kitti360_ezsp_0000000372_0000000610__center=-150_34_-2__radius=15_pred_colors.jpg}

\newcommand{\dalesInput}{images/prediction/dales_ezsp_5175_54395_new__TILE_3-3_OF_3-3__center=7_350_-3__radius=32_intensity_colors.jpg}
\newcommand{\dalesGT}{images/prediction/dales_ezsp_5175_54395_new__TILE_3-3_OF_3-3__center=7_350_-3__radius=32_y_colors.jpg}
\newcommand{\dalesPred}{images/prediction/dales_ezsp_5175_54395_new__TILE_3-3_OF_3-3__center=7_350_-3__radius=32_pred_colors.jpg}

\begin{tabular}{l@{}c@{}c@{}c@{}c}

\raiseandrotate{1.7}{\small S3DIS} 
&
\includegraphics[width=0.29\linewidth, height=0.2\linewidth]{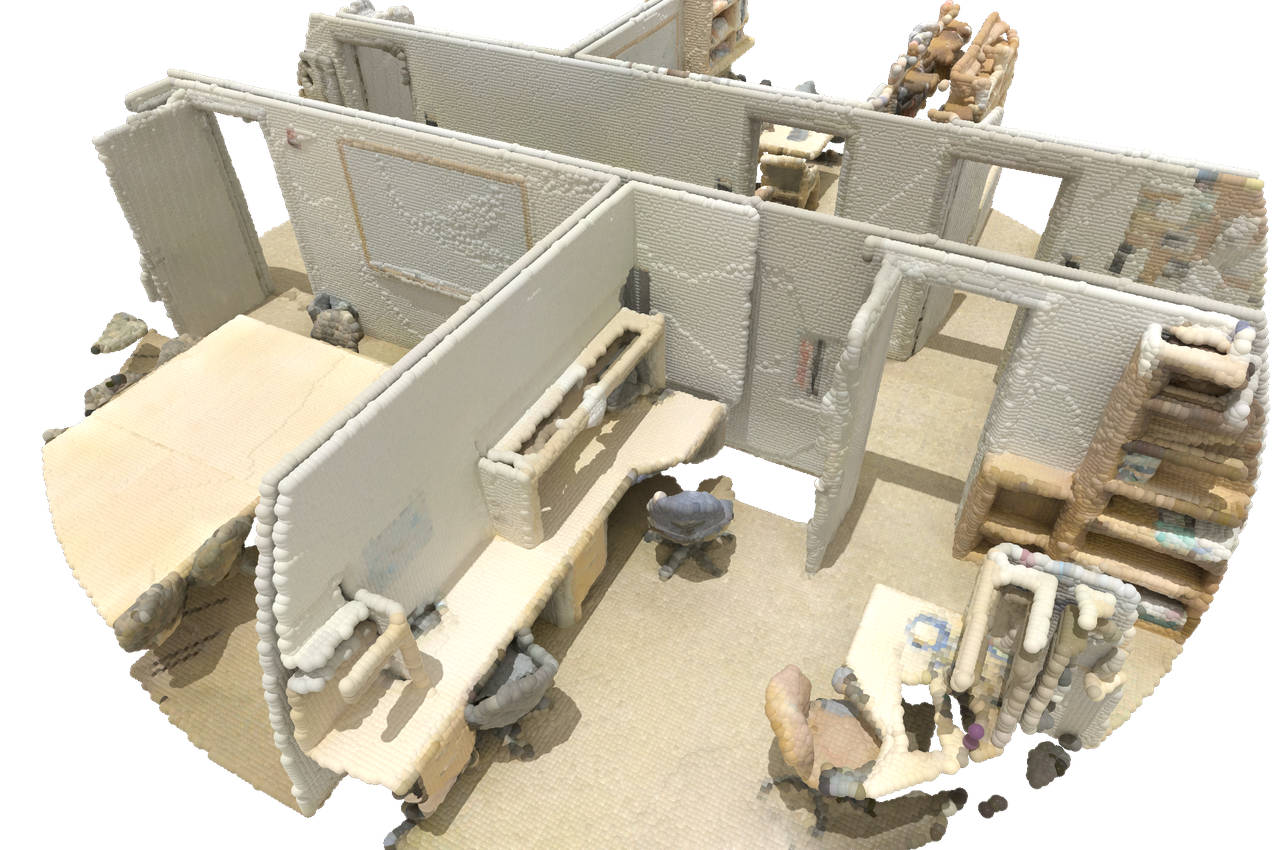}
&
\includegraphics[width=0.29\linewidth, height=0.2\linewidth]{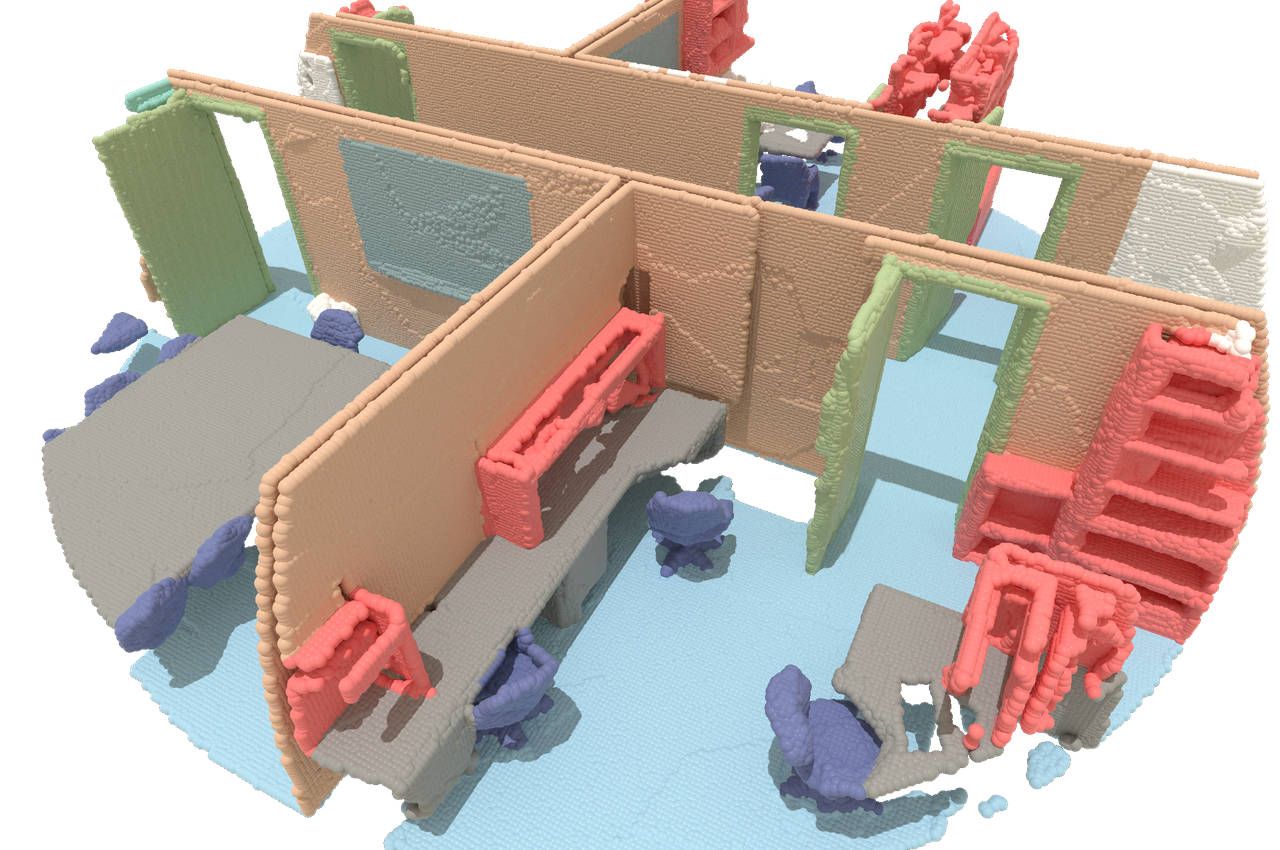}
&
\includegraphics[width=0.29\linewidth, height=0.2\linewidth]{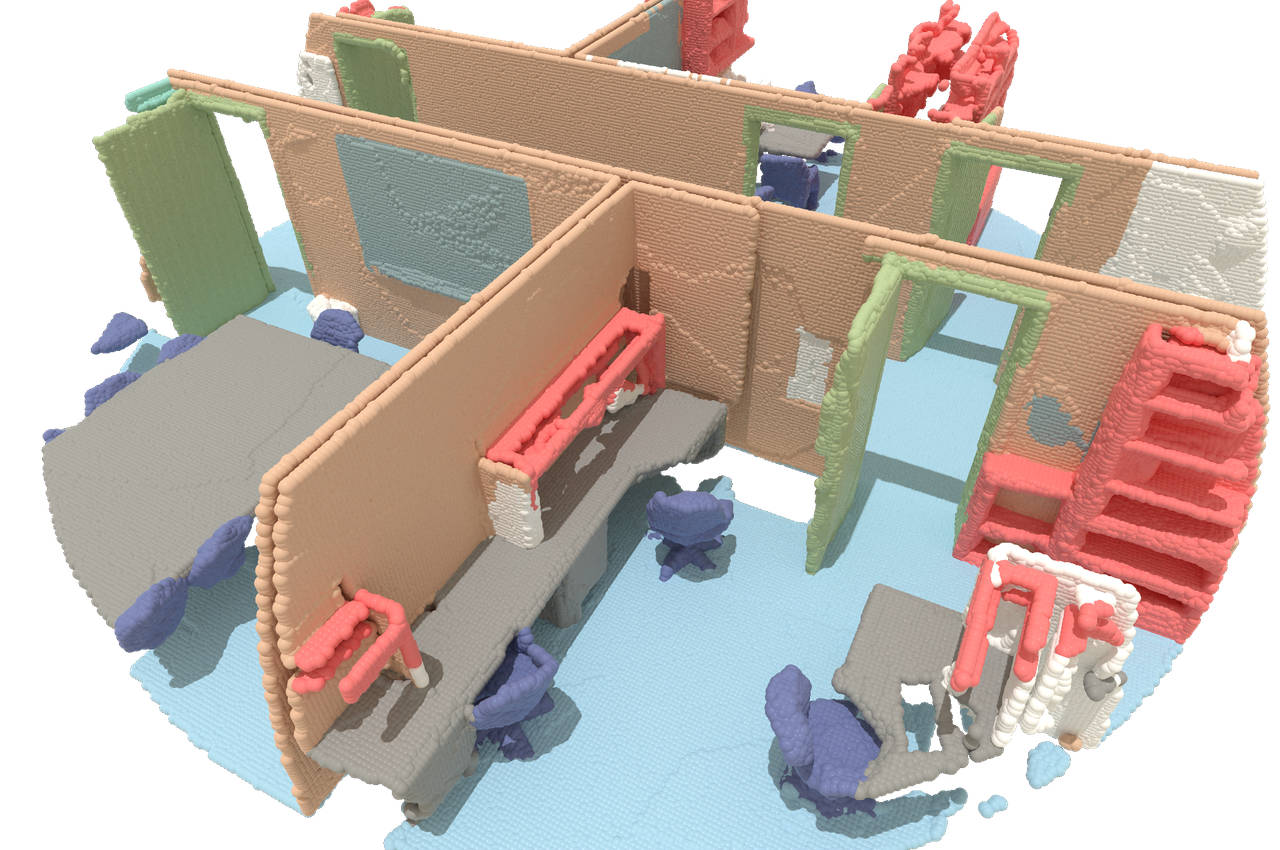}
&
\multirow{3}{*}[+3cm]
{
\begin{tabular}{@{\hspace{-2mm}}r@{\,}l}
          \definecolor{tempcolor}{rgb}{.37,0.61,0.77}
            \tikz[baseline=+0.3 ex] \fill[fill=tempcolor, scale=0.20,  draw=black] (0, 0) rectangle (1.2, 1.2); 
            & \scriptsize{floor}
            \\
            \definecolor{tempcolor}{rgb}{0.70,0.45,0.31}
            \tikz[baseline=+0.3 ex] \fill[fill=tempcolor, scale=0.20,  draw=black] (0, 0) rectangle (1.2, 1.2); 
            & \scriptsize{wall}
            \\
            \definecolor{tempcolor}{rgb}{0.31,0.63,.58}
            \tikz[baseline=+0.3 ex] \fill[fill=tempcolor, scale=0.20,  draw=black] (0, 0) rectangle (1.2, 1.2);
            & \scriptsize{column} 
            \\
            \definecolor{tempcolor}{rgb}{0.30,0.68,.32}
            \tikz[baseline=+0.3 ex] \fill[fill=tempcolor, scale=0.20,  draw=black] (0, 0) rectangle (1.2, 1.2); 
            & \scriptsize{wind.}
            \\
            \definecolor{tempcolor}{rgb}{.42,0.52,0.29}
            \tikz[baseline=+0.3 ex] \fill[fill=tempcolor, scale=0.20,  draw=black] (0, 0) rectangle (1.2, 1.2); 
            & \scriptsize{door}
            \\
            \definecolor{tempcolor}{rgb}{.16,0.19,0.39}
            \tikz[baseline=+0.3 ex] \fill[fill=tempcolor, scale=0.20,  draw=black] (0, 0) rectangle (1.2, 1.2);
            & \scriptsize{chair}
            \\
            \definecolor{tempcolor}{rgb}{.30,0.30,0.30}
            \tikz[baseline=+0.3 ex] \fill[fill=tempcolor, scale=0.20,  draw=black] (0, 0) rectangle (1.2, 1.2);
            & \scriptsize{table} 
            \\
            \definecolor{tempcolor}{rgb}{.88,0.20,0.20}
            \tikz[baseline=+0.3 ex] \fill[fill=tempcolor, scale=0.20,  draw=black] (0, 0) rectangle (1.2, 1.2); 
            & \scriptsize{bookcase}
            \\
            \definecolor{tempcolor}{rgb}{.32,0.43,.45}
            \tikz[baseline=+0.3 ex] \fill[fill=tempcolor, scale=0.20,  draw=black] (0, 0) rectangle (1.2, 1.2);
            & \scriptsize{board}
            \\
            \definecolor{tempcolor}{rgb}{0.91,0.91,.90}
            \tikz[baseline=+0.3 ex] \fill[fill=tempcolor, scale=0.20,  draw=black] (0, 0) rectangle (1.2, 1.2);
            & \scriptsize{clut.} 
            \\\midrule
             \definecolor{tempcolor}{rgb}{0.50, 0.25, 0.50}
            \tikz[baseline=+0.3 ex] \fill[fill=tempcolor, scale=0.20,  draw=black] (0, 0) rectangle (1.2, 1.2);
            & \scriptsize{road} 
            \\
            \definecolor{tempcolor}{rgb}{0.95, 0.13, 0.90}
            \tikz[baseline=+0.3 ex] \fill[fill=tempcolor, scale=0.20,  draw=black] (0, 0) rectangle (1.2, 1.2);
            & \scriptsize{sidewalk} 
            \\
            \definecolor{tempcolor}{rgb}{0.27, 0.27, 0.27}
            \tikz[baseline=+0.3 ex] \fill[fill=tempcolor, scale=0.20,  draw=black] (0, 0) rectangle (1.2, 1.2);
            & \scriptsize{building} 
            \\
            \definecolor{tempcolor}{rgb}{0.4 , 0.4 , 0.61}
            \tikz[baseline=+0.3 ex] \fill[fill=tempcolor, scale=0.20,  draw=black] (0, 0) rectangle (1.2, 1.2);
            & \scriptsize{wall} 
            \\
            \definecolor{tempcolor}{rgb}{0.74, 0.6 , 0.6 }
            \tikz[baseline=+0.3 ex] \fill[fill=tempcolor, scale=0.20,  draw=black] (0, 0) rectangle (1.2, 1.2);
            & \scriptsize{fence} 
            \\
            \definecolor{tempcolor}{rgb}{0.41, 0.55, 0.13}
            \tikz[baseline=+0.3 ex] \fill[fill=tempcolor, scale=0.20,  draw=black] (0, 0) rectangle (1.2, 1.2);
            & \scriptsize{veget.} 
            \\
            \definecolor{tempcolor}{rgb}{0.59, 0.98, 0.59}
            \tikz[baseline=+0.3 ex] \fill[fill=tempcolor, scale=0.20,  draw=black] (0, 0) rectangle (1.2, 1.2);
            & \scriptsize{terrain} 
            \\
            \definecolor{tempcolor}{rgb}{0.  , 0.  , 0.55}
            \tikz[baseline=+0.3 ex] \fill[fill=tempcolor, scale=0.20,  draw=black] (0, 0) rectangle (1.2, 1.2);
            & \scriptsize{car} 
            \\\midrule
            \definecolor{tempcolor}{rgb}{0.95, 0.84, 0.67}
            \tikz[baseline=+0.3 ex] \fill[fill=tempcolor, scale=0.20,  draw=black] (0, 0) rectangle (1.2, 1.2);
            & \scriptsize{ground} 
            \\
            \definecolor{tempcolor}{rgb}{0.27, 0.45, 0.26}
            \tikz[baseline=+0.3 ex] \fill[fill=tempcolor, scale=0.20,  draw=black] (0, 0) rectangle (1.2, 1.2); 
            & \scriptsize{veget.}
            \\
            \definecolor{tempcolor}{rgb}{0.91, 0.20, 0.94}
            \tikz[baseline=+0.3 ex] \fill[fill=tempcolor, scale=0.20,  draw=black] (0, 0) rectangle (1.2, 1.2); 
            & \scriptsize{car}
            \\
            \definecolor{tempcolor}{rgb}{0., 0.91,  0.04}
            \tikz[baseline=+0.3 ex] \fill[fill=tempcolor, scale=0.20,  draw=black] (0, 0) rectangle (1.2, 1.2); 
            & \scriptsize{fence}
            \\
            \definecolor{tempcolor}{rgb}{0.84, 0.26, 0.21}
            \tikz[baseline=+0.3 ex] \fill[fill=tempcolor, scale=0.20,  draw=black] (0, 0) rectangle (1.2, 1.2);
            & \scriptsize{building}
            \\ 
            \definecolor{tempcolor}{rgb}{0, 0, 0}
            \tikz[baseline=+0.3 ex] \fill[fill=tempcolor, scale=0.20,  draw=black] (0, 0) rectangle (1.2, 1.2);
            & \scriptsize{no label}  
            
\end{tabular}
}
\\
\raiseandrotate{1.1}{\small KITTI360} 
&
\includegraphics[width=0.29\linewidth, height=0.2\linewidth]{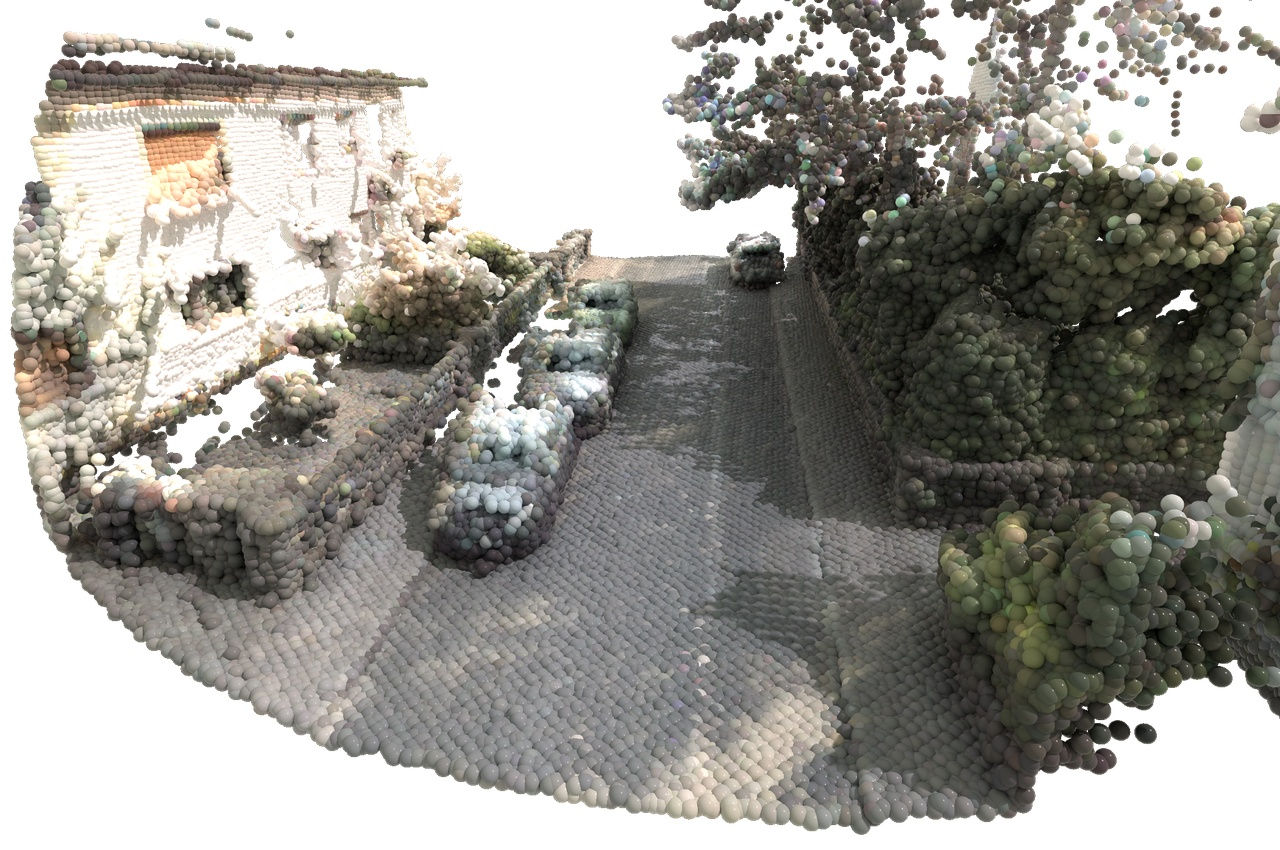}
&
\includegraphics[width=0.29\linewidth, height=0.2\linewidth]{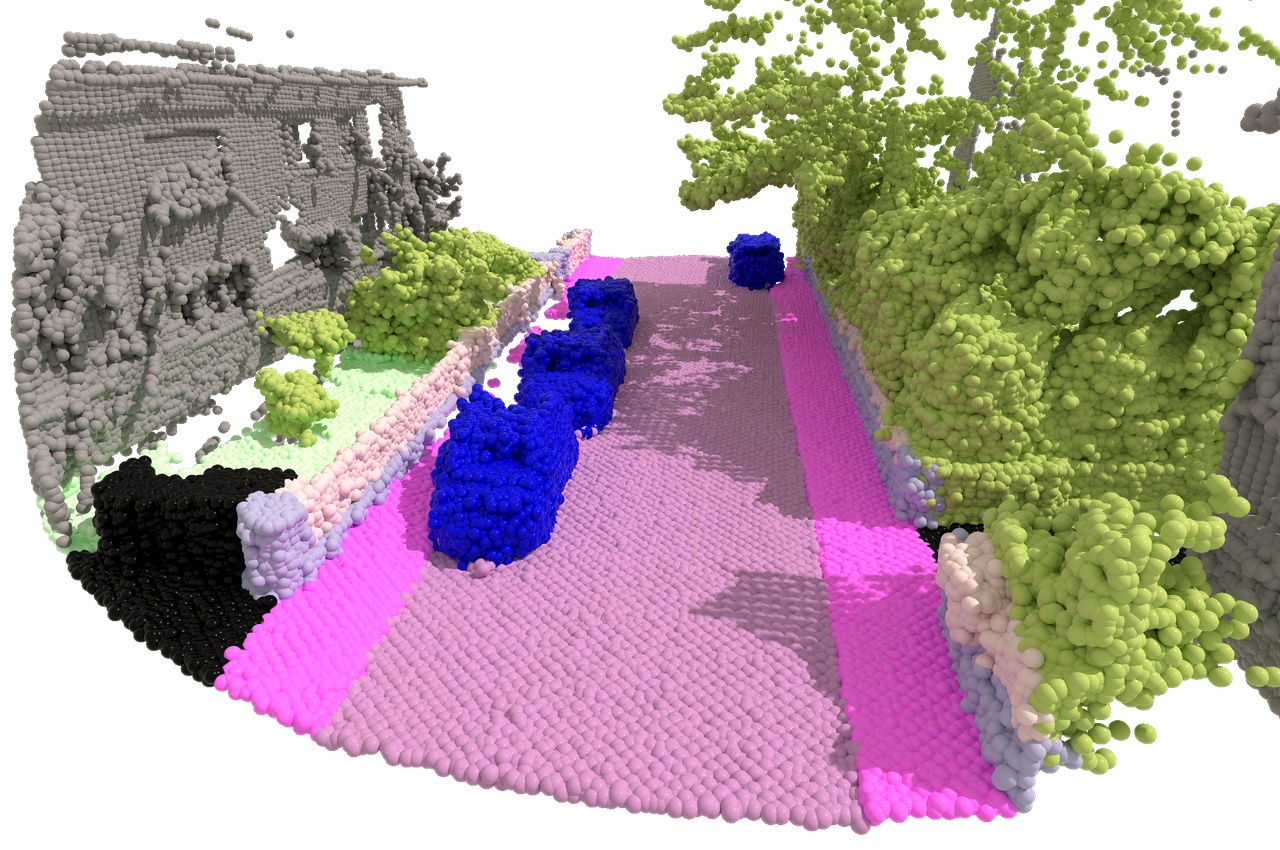}
&
\includegraphics[width=0.29\linewidth, height=0.2\linewidth]{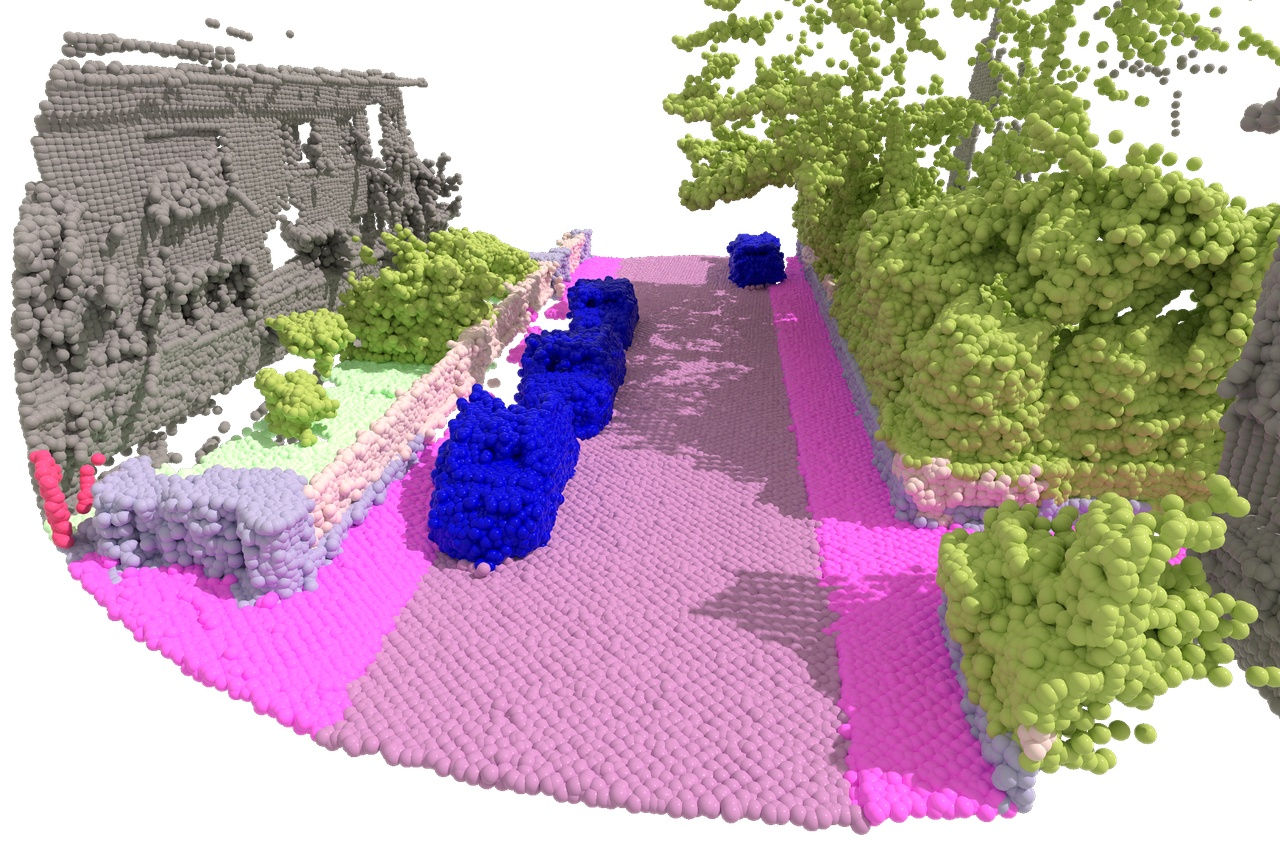}
\\
\raiseandrotate{1.2}{\small DALES} 
&
\includegraphics[width=0.29\linewidth, height=0.2\linewidth]{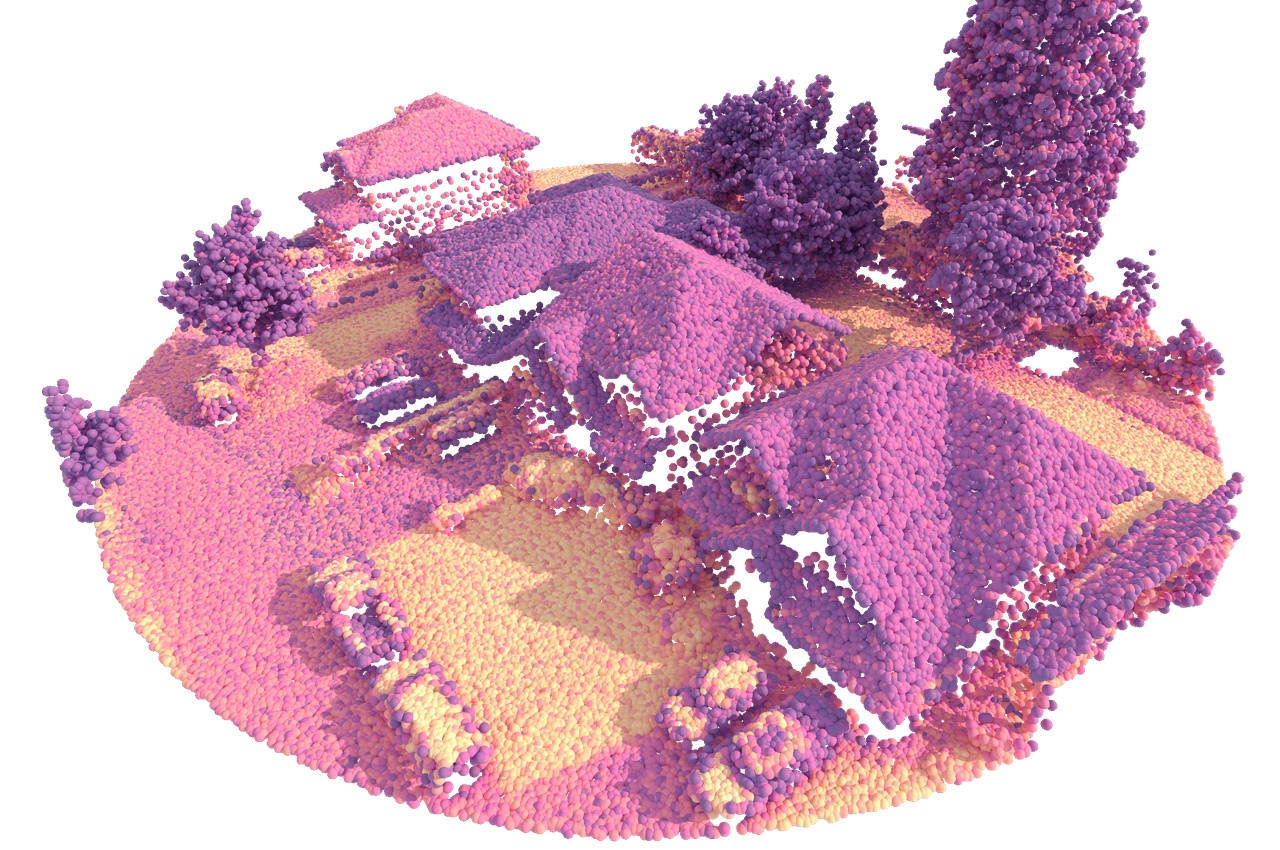}
&
\includegraphics[width=0.29\linewidth, height=0.2\linewidth]{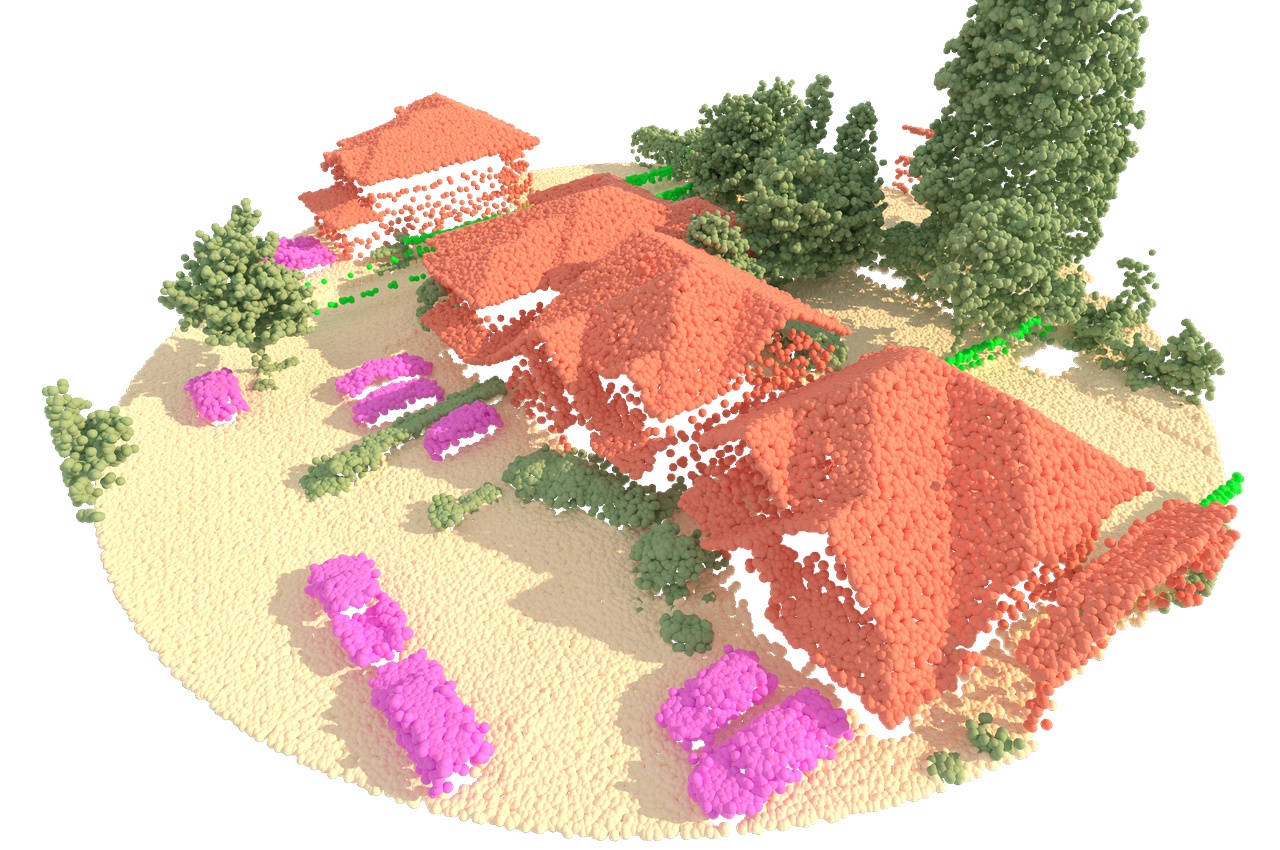}
&
\includegraphics[width=0.29\linewidth, height=0.2\linewidth]{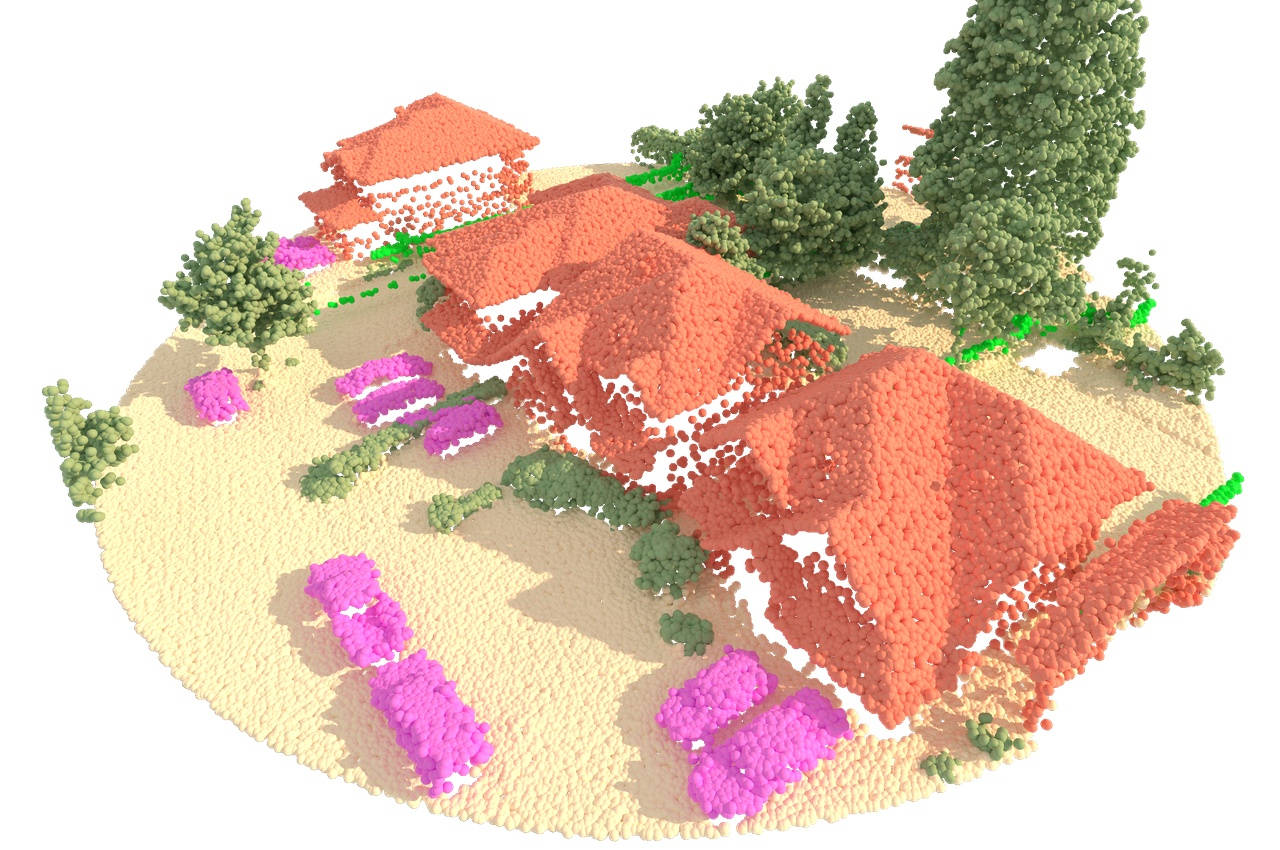}
\\
&
\begin{subfigure}{.25\linewidth}
    \caption{Input}
    \label{fig:sub:input}
\end{subfigure}
&
\begin{subfigure}{.25\linewidth}
    \caption{Ground Truth}
     \label{fig:sub:gt}
\end{subfigure}
&
\begin{subfigure}{.25\linewidth}
    \caption{Prediction}
     \label{fig:sub:pred}
\end{subfigure}
\end{tabular}

%% file: figures/bottleneck.tex
\begin{tikzpicture}
\begin{axis}[
    ybar stacked,
    bar width=30pt,
    ymin=0,
    ylabel={Time (s)},
    symbolic x coords={SPT, EZSP},
    xticklabels={SPT, EZ-SP},
    xtick=data,
    legend style={at={(0.5,1.02)}, anchor=south,legend columns=3},
    enlarge x limits=0.25,
    legend style={draw=none},
    ymajorgrids,
    height=7cm
]

\newdimen\HalfBar \HalfBar=15pt


\pgfmathsetmacro{\SPTvoxel}{1.9}
\pgfmathsetmacro{\EZSPvoxel}{1.9}

\pgfmathsetmacro{\SPTknn}{24.6}
\pgfmathsetmacro{\EZSPknn}{24.6}

\pgfmathsetmacro{\SPTpoint}{40.8}
\pgfmathsetmacro{\EZSPpoint}{1.8}

\pgfmathsetmacro{\SPTtransfer}{113.0}
\pgfmathsetmacro{\EZSPtransfer}{29.3}

\pgfmathsetmacro{\SPTpartition}{417.9}
\pgfmathsetmacro{\EZSPpartition}{3.4}

\pgfmathsetmacro{\SPTsp}{191.5}
\pgfmathsetmacro{\EZSPsp}{80.5}

\pgfmathsetmacro{\SPTinference}{13.7}
\pgfmathsetmacro{\EZSPinference}{13.7}

\pgfmathsetmacro{\SPTbaseVoxel}{0}                       
\pgfmathsetmacro{\SPTtopVoxel}{\SPTbaseVoxel+\SPTvoxel}

\pgfmathsetmacro{\SPTbaseKNN}{\SPTtopVoxel}              
\pgfmathsetmacro{\SPTtopKNN}{\SPTbaseKNN+\SPTknn}

\pgfmathsetmacro{\SPTbasePoint}{\SPTtopKNN}              
\pgfmathsetmacro{\SPTtopPoint}{\SPTbasePoint+\SPTpoint}

\pgfmathsetmacro{\SPTbaseTransfer}{\SPTtopPoint}         
\pgfmathsetmacro{\SPTtopTransfer}{\SPTbaseTransfer+\SPTtransfer}

\pgfmathsetmacro{\SPTbasePartition}{\SPTtopTransfer}     
\pgfmathsetmacro{\SPTtopPartition}{\SPTbasePartition+\SPTpartition}

\pgfmathsetmacro{\SPTbaseSP}{\SPTtopPartition}           
\pgfmathsetmacro{\SPTtopSP}{\SPTbaseSP+\SPTsp}

\pgfmathsetmacro{\SPTbaseInf}{\SPTtopSP}                 
\pgfmathsetmacro{\SPTtopInf}{\SPTbaseInf+\SPTinference}


\pgfmathsetmacro{\EZSPbaseVoxel}{0}                      
\pgfmathsetmacro{\EZSPtopVoxel}{\EZSPbaseVoxel+\EZSPvoxel}

\pgfmathsetmacro{\EZSPbaseKNN}{\EZSPtopVoxel}            
\pgfmathsetmacro{\EZSPtopKNN}{\EZSPbaseKNN+\EZSPknn}

\pgfmathsetmacro{\EZSPbasePoint}{\EZSPtopKNN}            
\pgfmathsetmacro{\EZSPtopPoint}{\EZSPbasePoint+\EZSPpoint}

\pgfmathsetmacro{\EZSPbaseTransfer}{\EZSPtopPoint}       
\pgfmathsetmacro{\EZSPtopTransfer}{\EZSPbaseTransfer+\EZSPtransfer}

\pgfmathsetmacro{\EZSPbasePartition}{\EZSPtopTransfer}   
\pgfmathsetmacro{\EZSPtopPartition}{\EZSPbasePartition+\EZSPpartition}

\pgfmathsetmacro{\EZSPbaseSP}{\EZSPtopPartition}         
\pgfmathsetmacro{\EZSPtopSP}{\EZSPbaseSP+\EZSPsp}

\pgfmathsetmacro{\EZSPbaseInf}{\EZSPtopSP}               
\pgfmathsetmacro{\EZSPtopInf}{\EZSPbaseInf+\EZSPinference}

\addplot+[fill=cKNN,draw=black, draw=none]        coordinates {(SPT,\SPTknn)        (EZSP,\EZSPknn)};
\addplot+[fill=cPoint,draw=black, draw=none]      coordinates {(SPT,\SPTpoint)      (EZSP,\EZSPpoint)};
\addplot+[fill=cTransfer,draw=black, draw=none]   coordinates {(SPT,\SPTtransfer)   (EZSP,\EZSPtransfer)};
\addplot+[fill=cPartition,draw=black, draw=none]  coordinates {(SPT,\SPTpartition)  (EZSP,\EZSPpartition)};
\addplot+[fill=cSP,draw=black, draw=none]         coordinates {(SPT,\SPTsp)         (EZSP,\EZSPsp)};
\addplot+[fill=cInference,draw=black, draw=none]  coordinates {(SPT,\SPTinference)  (EZSP,\EZSPinference)};



\path[fill=cKNN, fill opacity=0.35, draw=cKNN, draw=none]
  ([xshift=\HalfBar, yshift=-0.35pt]axis cs:SPT,\SPTbaseKNN)
  -- ([xshift=\HalfBar, yshift=-0.35pt]axis cs:SPT,\SPTtopKNN)
  to[out=0, in=180]   ([xshift=-\HalfBar, yshift=-0.35pt]axis cs:EZSP,\EZSPtopKNN)
  -- ([xshift=-\HalfBar, yshift=-0.35pt]axis cs:EZSP,\EZSPbaseKNN)
  to[out=180, in=0]   ([xshift=\HalfBar, yshift=-0.35pt]axis cs:SPT,\SPTbaseKNN)
  -- cycle;

\path[fill=cPoint, fill opacity=0.35, draw=cPoint, draw=none]
  ([xshift=\HalfBar, yshift=-0.35pt]axis cs:SPT,\SPTbasePoint)
  -- ([xshift=\HalfBar, yshift=-0.35pt]axis cs:SPT,\SPTtopPoint)
  to[out=0, in=180]   ([xshift=-\HalfBar, yshift=-0.35pt]axis cs:EZSP,\EZSPtopPoint)
  -- ([xshift=-\HalfBar, yshift=-0.35pt]axis cs:EZSP,\EZSPbasePoint)
  to[out=180, in=0]   ([xshift=\HalfBar, yshift=-0.35pt]axis cs:SPT,\SPTbasePoint)
  -- cycle;

\path[fill=cTransfer, fill opacity=0.35, draw=cTransfer, draw=none]
  ([xshift=\HalfBar, yshift=-0.35pt]axis cs:SPT,\SPTbaseTransfer)
  -- ([xshift=\HalfBar, yshift=-0.35pt]axis cs:SPT,\SPTtopTransfer)
  to[out=0, in=180]   ([xshift=-\HalfBar, yshift=-0.35pt]axis cs:EZSP,\EZSPtopTransfer)
  -- ([xshift=-\HalfBar, yshift=-0.35pt]axis cs:EZSP,\EZSPbaseTransfer)
  to[out=180, in=0]   ([xshift=\HalfBar, yshift=-0.35pt]axis cs:SPT,\SPTbaseTransfer)
  -- cycle;

\path[fill=cPartition, fill opacity=0.35, draw=cPartition, draw=none]
  ([xshift=\HalfBar, yshift=-0.35pt]axis cs:SPT,\SPTbasePartition)
  -- ([xshift=\HalfBar, yshift=-0.35pt]axis cs:SPT,\SPTtopPartition)
  to[out=0, in=180]   ([xshift=-\HalfBar, yshift=-0.35pt]axis cs:EZSP,\EZSPtopPartition)
  -- ([xshift=-\HalfBar, yshift=-0.35pt]axis cs:EZSP,\EZSPbasePartition)
  to[out=180, in=0]   ([xshift=\HalfBar, yshift=-0.35pt]axis cs:SPT,\SPTbasePartition)
  -- cycle;

\path[fill=cSP, fill opacity=0.35, draw=cSP, draw=none]
  ([xshift=\HalfBar, yshift=-0.35pt]axis cs:SPT,\SPTbaseSP)
  -- ([xshift=\HalfBar, yshift=-0.35pt]axis cs:SPT,\SPTtopSP)
  to[out=0, in=180]   ([xshift=-\HalfBar, yshift=-0.35pt]axis cs:EZSP,\EZSPtopSP)
  -- ([xshift=-\HalfBar, yshift=-0.35pt]axis cs:EZSP,\EZSPbaseSP)
  to[out=180, in=0]   ([xshift=\HalfBar, yshift=-0.35pt]axis cs:SPT,\SPTbaseSP)
  -- cycle;

\path[fill=cInference, fill opacity=0.35, draw=cInference, draw=none]
  ([xshift=\HalfBar, yshift=-0.35pt]axis cs:SPT,\SPTbaseInf)
  -- ([xshift=\HalfBar, yshift=-0.35pt]axis cs:SPT,\SPTtopInf)
  to[out=0, in=180]   ([xshift=-\HalfBar, yshift=-0.35pt]axis cs:EZSP,\EZSPtopInf)
  -- ([xshift=-\HalfBar, yshift=-0.35pt]axis cs:EZSP,\EZSPbaseInf)
  to[out=180, in=0]   ([xshift=\HalfBar, yshift=-0.35pt]axis cs:SPT,\SPTbaseInf)
  -- cycle;

\legend{\small KNN, \small Point features, \small CPU--GPU transfer, \small Partition, \small Superpoint features,  \small Inference}
\end{axis}
\end{tikzpicture}

%% file: tables/spectrum.tex
\resizebox{\linewidth}{!}{
\begin{tabular}{c c c c }
\toprule
\multirow{1}{*}{Scenario} & Points & VRAM & Compatible Hardware \\\midrule
\makecell{{\bf Autonomous driving} \\single LiDAR scan} & \makecell{105k pts} & \makecell{\MODELNAME: 0.25 GB \\ MinkowskiNet: 0.52 GB} & \makecell{Jetson-Nano \\    } \\\greyrule
\makecell{{\bf Digital twin} \\building-scale} & \makecell{79M pts} & \makecell{\MODELNAME: 29 GB  \\ MinkowskiNet: 30 GB} & \makecell{2$\times$ Radeon RX  \\ 7900: 40GB} \\\greyrule
\makecell{{\bf Aerial survey} \\city-scale, 1.3\,km\textsuperscript{2}} & \makecell{16M pts} & \makecell{\MODELNAME: 45 GB \\ MinkowskiNet : OoM} & \makecell{A40 \\ 48GB} \\
\bottomrule
\end{tabular}
}

%% file: tables/ablation.tex
\begin{tabular}{llcc} \toprule
&\multirow{2}{*}[0mm]{Configuration} & S3DIS Fold5 & Throughput\\
&& mIoU & $\times 10^6$ pt/s \\\midrule 
&\bf Best  & $69.6$ & $1.7$ \\ \greyrule
\textsc{A}& Handcrafted features   & $69.5$ & $1.7$ \\
\textsc{B}& 2 hierarchical levels   & $67.2$ & $1.7$ \\
\textsc{C}& No optimization  & -  & $0.8$ \\
\textsc{D}& PCP & $67.8$ & $0.3$ \\
\end{tabular}

%% file: sections/5_acknowledgements.tex
\para{Acknowledgements.} 
This work was supported by ANR project ANR-23-PEIA-0008,
and was granted access to the HPC resources of IDRIS under the allocation  AD011015519R1 and AD011015196R1 made by GENCI.
We thank Davide Allegro for inspiring discussions and valuable feedback.

%% file: main.bbl

%% file: sections/appendix.tex
\FloatBarrier
\pagebreak
\clearpage

\section*{\centering \LARGE Appendix}
\setcounter{section}{0}
\setcounter{figure}{0}
\setcounter{table}{0}
\renewcommand*{\thesection}{appendix.\the\value{section}}
\renewcommand\thefigure{A-\arabic{figure}}
\renewcommand\thesection{A-\arabic{section}}
\renewcommand\thetable{A-\arabic{table}}
\renewcommand\theequation{A-\arabic{equation}}
\newcommand\thealgorithm{A-\arabic{algorithm}}

In this document, we prove Proposition~\ref{prop:merge} (\ref{sec:proof}), provide the pseudo for our partition algorithm (\ref{sec:pseudocode}), detail how we measure the computation breakdown across all methods (\ref{sec:benchmark_speed}), and provide further implementation details (\ref{sec:implem_details_partition}).

\section{Proof of Proposition~\ref{prop:merge} }
\label{sec:proof}


\begin{prop}
Merging adjacent superpoints $(P,Q) \in \SEdges$ decreases $\Omega(\Partition)$ by the following \emph{edge merge gain:}
\begin{align}\label{eq:delta}
    \Delta(P,Q) = -\frac{|P||Q|}{|P|+|Q|}\|\Feat_P - \Feat_Q\|^2 + \regul\!\!\!\!\!\sum_{(p,q)\in(P \times Q)\cap\Edges}\!\!\!\!\!\!\!\!\!\!\!\edgeweights_{p,q}~.
\end{align}
\label{prop:merge}
\end{prop}

\begin{proof}
We can write the energy of the partition $\Partition$ as:
\begin{align}
\Omega(\Partition)
&=
\sum_{U \in \Partition} 
    \sum_{p \in U} 
    \Vert \feat_p - \Feat_U \Vert^2
+
\regul
\sum_{(p,q)\in \Edges}
    \edgeweights_{p,q}
    \Vert \feat_p - \feat_q \Vert_0 \\
&=
\sum_{U \in \Partition} \left( 
\sum_{p \in U} \Vert \feat_p   \Vert^2
- \vert U \vert \Vert \Feat_U \Vert^2 
\right) 
+
\regul
\sum_{(U,V)\in \SEdges}
    \edgeweights_{U,V}~,
\end{align}

where $\edgeweights_{U,V}=\sum_{(p,q)\in(U \times V)\cap\Edges} \edgeweights_{p,q}~$.

We define $\Partition’$ the modified partition in which $P$ and $Q$ are merged:
\begin{align}\nonumber
\Omega(\Partition')
&=
\sum_{U \in \Partition'} \left( 
\sum_{p \in U} \Vert \feat_p   \Vert^2
- \vert U \vert \Vert \Feat_U \Vert^2 
\right) 
+\\
&\regul
\left( 
\sum_{(U,V)\in \SEdges}
    \edgeweights_{U,V} 
- \edgeweights_{P,Q}
\right)~.
\end{align}

We denote the new superpoint $R$ obtained by the merge of $P$ and $Q$:
\begin{align}\nonumber
    \Feat_R 
    &= \frac1{\vert P  \vert + \vert Q \vert}
    \sum_{p \in P \cup Q} \feat_p \\
    &=\frac
    {\vert P \vert \Feat_P + \vert Q \vert \Feat_Q}
    {\vert P  \vert + \vert Q \vert}~.
\end{align}

We compute $\Delta(P,Q)$ the gain in energy of this merge:
\begin{align}\nonumber
    \Delta(P,Q) 
    &= \Omega(\Partition) - \Omega(\Partition')\\\nonumber
    &=  
   \vert R \vert \Vert \Feat_R \Vert^2 - 
    \sum_{p \in P \cup Q} \Vert \feat_p  \Vert^2 \\\nonumber
    &-  \vert P \vert  \Vert \Feat_P \Vert^2 -  \vert Q \vert  \Vert \Feat_Q \Vert^2
    + \sum_{p \in P \cup Q} \Vert \feat_p  \Vert^2 \\\nonumber
    & + \regul w_{P,Q} \\
    & =\vert R \vert \Vert \Feat_R \Vert^2 -\vert P \vert \Vert \Feat_P \Vert^2 - \vert Q \vert \Vert \Feat_Q \Vert^2  + \regul w_{P,Q} \\
    & =-\frac{\vert P \vert \vert Q \vert}{\vert P \vert + \vert Q \vert} \Vert \Feat_P - \Feat_Q \Vert^2 + \regul w_{P,Q}~.
\end{align}

\end{proof}

\input{algorithms/wcc}

\section{GPU-based Partition Implementation}
\label{sec:pseudocode}

We provide here the pseudocode of our parallelized greedy partition algorithm introduced in \secref{subsec:transition}. 
Our approach relies on two recursive algorithms: \algoref{algo:wcc} computing the weakly connected components of a graph and \algoref{algo:merge} merging nodes of a graph based on their energy.

We implemented both algorithms relying solely on CUDA-accelerated PyTorch operations.
Note that both algorithms are recursive, which significantly reduces complexity.


\section{Benchmarking Processing Speeds}
\label{sec:benchmark_speed}

We detail here how the S3DIS preprocessing, partition, and segmentation speeds communicated in \cref{fig:over} and \cref{tab:bigtable} were obtained.
Whenever the gathered preprocessing times also included the times for reading raw dataset files from disk, we subtracted the time for our own file reading, to focus on the actual inference speed of each method.
Whenever we measured throughput using official public code, and unless specified otherwise, we used a machine with an NVIDIA V100 GPU with 32G VRAM, a 10-core CPU, and 94G RAM.  

\para{PointNet++, KPConv, and MinkowskiNet.}
We used the Torch-Points3D~\cite{chaton2020torch} implementation and public logs~\cite{githubpointnetpp, tp3dpointnetpp, tp3dkpconv, tp3dminko} for estimating the preprocessing and inference times on S3DIS 6-fold for PointNet++~\cite{qi2017pointnetpp}, KPConv~\cite{thomas2019kpconv}, and MinkowskiNet~\cite{choy20194d}.

\para{PointNeXt.}
We used the official logs~\cite{pointnextlogs} to recover the preprocessing and inference times for PointNeXt-XL~\cite{qian2022pointnext} on S3DIS 6-fold. 
Based on the logs and the official codebase, the voxelization is done on the fly during training and inference, which explains why PointNeXt has essentially no preprocessing time in \cref{tab:bigtable}.

\input{algorithms/merge}

\para{Stratified Transformer.}
We used the training logs provided in the official GiHub repository~\cite{githubstratifiedtransformer} to gather the inference time for Stratified Transformer~\cite{lai2022stratified}. 
The Stratified Transformer codebase points to the PointNet++ implementation~\cite{githubpointnetpp} for preprocessing S3DIS, which explains why these models share the same preprocessing times in \cref{tab:bigtable}.
Regarding inference times, the training and inference logs~\cite{githubstratifiedtransformer} illustrate that this method uses about $\times60$ test-time augmentations to produce the results communicated in the official publication. 
While Stratified Transformer might achieve acceptable performance and higher throughput without these multiple predictions, the fact that the communicated performance was assessed using test-time augmentations does not allow disentangling this step from the whole pipeline.

\para{Point Transformer v3.}
We ran the S3DIS preprocessing script from the official Point Transformer v3~\cite{wu2024point} implementation~\cite{githubptv3} to measure preprocessing times.
For measuring inference times, we referred to the official logs~\cite{huggingfaceptv3}.
Similar to Stratified Transformer, it is worth noting that the Point Transformer v3 implementation uses about $\times 200$ ensembled predictions with test-time augmentations, running on $4$ GPUs in its final paper, which explains the very low throughput of this approach.

\para{Superpoint Graph and SSP.}
We ran the code from the official repository~\cite{githubspg} to measure preprocessing, partition, and inference times of Superpoint Graph~\cite{landrieu2018spg} and SSP~\cite{landrieu2019point}.

\para{SPNet.}
We used the official implementation~\cite{githubspnet} of SPNet~\cite{hui2021superpoint} to compute preprocessing, partition, and inference times.
Since the codebase largely relies on the SPG~\cite{githubspg} implementation, some preprocessing steps are identical.
Note that the current state of the codebase only allowed us to run the partition on GPU architectures that tolerated the project CUDA and PyTorch dependencies.
This prevented experimentation on hardware with larger VRAM.
For this reason, our results in \cref{fig:over} were limited by GPU memory.
Still, these allow seeing the trend of the SPNet oracle mIoU against \MODELNAME~as the number for superpoint grows. 

\para{Superpoint Transformer.}
We used the official implementation~\cite{githubspt} for Superpoint Transformer~\cite{robert2023efficient} to measure preprocessing, partition, and inference times.

\section{Further Implementation details}
\label{sec:implem_details_partition}
We provide further details of the implementation for detecting the semantic transition \ref{subsec:transition} and training the semantic segmentation \ref{subsec:classification}.

\para{Detecting the semantic transition.}
We employ a focal loss~\cite{lin2017focal} with $\gamma=1$ to learn the semantic boundaries in all datasets.
The learning rate is set to $10^{-4}$ for S3DIS and $5.10^{-4}$ for KITTI-360 and DALES without any scheduler.
We set the weight decay to $10^{-4}$.

\para{Semantic segmentation.}
Since the point embeddings produced by our CNN have a higher dimensionality than the handcrafted features in the official SPT~\cite{robert2023efficient} implementation, we modify the first layer of the PointNet-like MLP of SPT from $[32, 64, 128]$ to $[48, 64, 128]$ channels.

%% file: algorithms/wcc.tex
\begin{algorithm}[h]
\DontPrintSemicolon
\SetKwFunction{WCC}{wcc\_max\_prop}
\SetKwInOut{Input}{Input\hphantom{}}
\SetKwInOut{Output}{Output\hphantom{}}

\newcommand{\AlgoVariable}[1]{\textcolor{blue!70!black}{#1}}
\newcommand{\AlgoComment}[1]{\textcolor{DraculaGreen!50!black}{\tcp{#1}}}
\newcommand{\AlgoFunction}[1]{\textcolor{DraculaPurple!70!black}{\texttt{#1}}}
\newcommand{\AlgoFunctionNative}[1]{\textcolor{DraculaPurple!80!black}{\mathtt{#1}}}
\newcommand{\AlgoKeywordNative}[1]{\textcolor{DraculaOrange!80!black}{{#1}}}

\SetKwIF{If}{ElseIf}{Else}{\AlgoKeywordNative{if}}{\AlgoKeywordNative{then}}{\AlgoKeywordNative{else if}}{\AlgoKeywordNative{else}}%

\SetKwFor{For}{\AlgoKeywordNative{for}}{\AlgoKeywordNative{do}}{}
\SetKwFor{While}{\AlgoKeywordNative{while}}{\AlgoKeywordNative{do}}{}
\SetKw{Return}{\AlgoKeywordNative{return}}

\Input{Number of nodes $N$, undirected edges $\Edges$}
\Output{Component assignment vector $\Indices \in \{0, \dots, C{-}1\}^N$}

\SetKwProg{Fn}{Function}{:}{}
\Fn{\WCC{$N$, $\Edges$}}{

    \AlgoComment{random index per node}
    $\Indices \leftarrow \AlgoFunction{randperm}(N)$ \;

    \AlgoComment{max-pool on neighbors}
    $\Indices_{\text{max}} \leftarrow \AlgoFunction{max\_propagation}(\Indices, \Edges)$ \;
    
    \If{$\Indices_{\text{max}} = \Indices$}{
        \Return $\Indices$ \;
    }
    $\Indices_{\text{consec}} \leftarrow \AlgoFunction{to\_consecutive\_ids}(\Indices_{\text{max}})$ \;

    \AlgoComment{merge nodes with same $\Indices_{\text{consec}}$}
    $\SEdges_{\text{comp}} \leftarrow \AlgoFunction{component\_graph}(\Indices_{\text{consec}}, \Edges)$ \;
    
    $N_{\text{comp}} \leftarrow \AlgoFunction{max}(\Indices_{\text{consec}}) + 1$ \;

    \AlgoComment{recursive call}
    $\Indices_{\text{rec}} \leftarrow \AlgoFunction{\WCC}(N_{\text{comp}}, \SEdges_{\text{comp}})$ \;

    \AlgoComment{distribute component labels}
    $\Indices_{\text{comp}} \leftarrow \Indices_{\text{rec}}[\Indices_{\text{consec}}]$ \;

    \Return $\Indices_{\text{comp}}$
}
\caption{Weakly Connected Components by Max Propagation}
\label{algo:wcc}
\end{algorithm}

%% file: algorithms/merge.tex
\begin{algorithm}[b!]
\DontPrintSemicolon
\SetKwFunction{Merge}{merge}
\SetKwInOut{Input}{Input}
\SetKwInOut{Parameters}{Parameters}
\SetKwInOut{Output}{Output}

\newcommand{\AlgoVariable}[1]{\textcolor{blue!70!black}{#1}}
\newcommand{\AlgoComment}[1]{\textcolor{DraculaGreen!50!black}{\tcp{#1}}}
\newcommand{\AlgoFunction}[1]{\textcolor{DraculaPurple!70!black}{\texttt{#1}}}
\newcommand{\AlgoFunctionNative}[1]{\textcolor{DraculaPurple!80!black}{\mathtt{#1}}}
\newcommand{\AlgoKeywordNative}[1]{\textcolor{DraculaOrange!80!black}{{#1}}}

\SetKwIF{If}{ElseIf}{Else}{\AlgoKeywordNative{if}}{\AlgoKeywordNative{then}}{\AlgoKeywordNative{else if}}{\AlgoKeywordNative{else}}%

\SetKwFor{For}{\AlgoKeywordNative{for}}{\AlgoKeywordNative{do}}{}
\SetKwFor{While}{\AlgoKeywordNative{while}}{\AlgoKeywordNative{do}}{}
\SetKw{Return}{\AlgoKeywordNative{return}}

\Input{
    Node features $\Feat$, sizes $\Size$, positions $\Coords$,\\
    undirected graph edges $\SEdges$ and \\weights $\Edgeweights$ 
}
\Parameters{
    Regularization $\regul$, min size $\sizemin$, nearest neighbor $k$
}
\Output{
    Component assignment $\Indices \in \{0, \dots, C{-}1\}^N$
}

\SetKwProg{Fn}{Function}{:}{}
\Fn{\Merge{$\Feat, \Size, \Coords, \SEdges, \Edgeweights, \regul, \sizemin, k$}}{

    $N \leftarrow |\Feat|$
    

    
    
    

    \AlgoComment{remove self-loops, duplicates, and connect isolated nodes}
    $\SEdges, \Edgeweights \leftarrow \AlgoFunction{prepare\_graph}(\SEdges, \Edgeweights, \Coords, k)$ \;
    
    \If{$|\SEdges| = 0$}{
      \Return $\Indices \leftarrow [0, 1, \dots, N{-}1]$ \;
    }
    
    \AlgoComment{merging energy, see \cref{eq:delta}}
    $\Delta \leftarrow \AlgoFunction{edge\_merge\_energy}(\Feat, \Size, \SEdges, \Edgeweights, \regul)$ \;

    
    


    \AlgoComment{select directed $(P \rightarrow Q)$ edges if $|P|<\sizemin$, or $\Delta(P \rightarrow Q)>0$, only the best $(P \rightarrow \bullet)$ with highest $\Delta$ is kept}
    $\SEdges_\text{merge} \leftarrow \AlgoFunction{keep\_valid\_merges}(\SEdges, \Delta, \sizemin)$ \;

    \AlgoComment{no valid merge}
    \If{$|\SEdges_\text{merge}| = 0$}{
      \Return $\Indices \leftarrow [0, 1, \dots, N{-}1]$ \;
    }
    
    \AlgoComment{identify connected components in the merge graph, \eg, $\{(P \rightarrow R), (Q \rightarrow R)\} \in \SEdges_\text{merge}$}
    $\Indices_\text{comp} \leftarrow \AlgoFunction{wcc\_max\_prop}(N, \SEdges_\text{merge})$ \;

    \AlgoComment{only one component left}
    \If{$\AlgoFunctionNative{max}(\Indices_\text{comp}) = 0$}{
      \Return $\Indices_\text{comp}$ \;
    }

    
    
    
    \AlgoComment{node and edge attributes of the new components}
    $\Feat_\text{comp}, \Size_\text{comp}, \Coords_\text{comp} \leftarrow \AlgoFunction{component\_attributes}(\Indices_\text{comp}, \Feat, \Size, \Coords)$ \;
    
    $\SEdges_\text{comp}, \Edgeweights_\text{comp} \leftarrow \AlgoFunction{component\_graph}(\Indices_\text{comp}, \SEdges, \Edgeweights)$ \;

    \AlgoComment{recursive call}
    $\Indices_\text{rec} \leftarrow \AlgoFunction{\Merge}($
      $\Feat_\text{comp}, \Size_\text{comp}, \Coords_\text{comp}, \SEdges_\text{comp}, \Edgeweights_\text{comp}, \regul, \sizemin, k)$ \;
    
    \Return $\Indices_\text{rec}[\Indices_\text{comp}]$ \;
}
\caption{Component Merging with Low Contour Prior}
\label{algo:merge}
\end{algorithm}